\documentclass[10pt,journal,compsoc]{IEEEtran}
\ifCLASSOPTIONcompsoc
  \usepackage{cite}
\else
  \usepackage{cite}
\fi
\ifCLASSINFOpdf
\else
\fi
\usepackage{amsmath}
\usepackage{algorithmic}
\usepackage{graphicx}
\usepackage{subfigure}
\usepackage{booktabs} % for professional tables
\usepackage{amsthm}
\usepackage{amssymb}
\usepackage{bm}
\usepackage{newtxmath}
\usepackage{xcolor}
\usepackage{xspace}
\usepackage{bbm}
\usepackage{algorithm}
\usepackage{algorithmic}
\usepackage[nolist]{acronym}
\usepackage{hyperref}
\usepackage[shortlabels]{enumitem}

\hypersetup{
    colorlinks=false,
    linkcolor=black,
    filecolor=black,      
    urlcolor=blue,
    }

\usepackage{amsmath,amsfonts,bm}

\def\eqref#1{equation~\ref{#1}}
\def\Eqref#1{Equation~\ref{#1}}
\def\1{\bm{1}}

\def\vmu{{\bm{\mu}}}
\def\vtheta{{\bm{\theta}}}
\def\va{{\bm{a}}}

\def\vo{{\bm{o}}}

\def\vu{{\bm{u}}}

\def\vw{{\bm{w}}}
\def\vx{{\bm{x}}}
\def\vy{{\bm{y}}}
\def\vz{{\bm{z}}}

\def\mB{{\bm{B}}}

\def\mL{{\bm{L}}}

\DeclareMathAlphabet{\mathsfit}{\encodingdefault}{\sfdefault}{m}{sl}
\SetMathAlphabet{\mathsfit}{bold}{\encodingdefault}{\sfdefault}{bx}{n}

\DeclareMathOperator*{\argmax}{arg\,max}
\DeclareMathOperator*{\argmin}{arg\,min}

\DeclareMathOperator{\sign}{sign}

\theoremstyle{plain}
\newtheorem{theorem}{Theorem}
\newtheorem*{theorem*}{Theorem}
\newenvironment{hproof}{\proof}{\endproof}
\def\vxi{{\bm{\xi}}}
\def\vmu{{\bm{\mu}}}
\def\vnu{{\bm{\nu}}}
\def\vDelta{{\bm{\Delta}}}
\def\vlambda{{\bm{\lambda}}}
\DeclareMathOperator{\relu}{relu}

\begin{document}
\title{Continuous-Time Fitted Value Iteration\\ for Robust Policies}

\author{Michael~Lutter, Boris Belousov, 
Shie~Mannor,
Dieter~Fox,
Animesh~Garg, 
and~Jan~Peters% <-this % stops a space
\IEEEcompsocitemizethanks{
\IEEEcompsocthanksitem M. Lutter is with the Intelligent Autonomous Systems group with the Computer Science Department, Technical University of Darmstadt, Darmstadt, Germany. % \protect\\
E-mail: michael@robot-learning.de
\IEEEcompsocthanksitem B. Belousov is with the Intelligent Autonomous Systems group with the Computer Science Department, Technical University of Darmstadt, Darmstadt, Germany. % \protect\\
\IEEEcompsocthanksitem S. Mannor is a professor in the The Andrew and Erna Viterbi Faculty of Electrical \& Computer Engineering at the Technion – Israel Institute of Technology and a Distinguished Scientist at NVIDIA, Tel Aviv, Israel.
\IEEEcompsocthanksitem D. Fox is a professor in the Allen School of Computer Science \& Engineering at the University of Washington and Senior Director of Robotics Research at NVIDIA, Seattle, Washington, USA. 
\IEEEcompsocthanksitem A. Garg is a professor in the Computer Science Department at the University of Toronto, Member of the Vector Institute and Senior Scientist at NVIDIA, Toronto, Canada. 
\IEEEcompsocthanksitem J. Peters is a professor at the Intelligent Autonomous Systems group with the Computer Science Department, Technical University of Darmstadt, Darmstadt, Germany. % \protect\\
}% <-this % stops an unwanted space
\thanks{
% Manuscript received August 8, 2021;
% revised August 26, 2015.
}}


\markboth{
%IEEE Transaction on Pattern Analysis and Machine Intelligence, August~2021
}%
{Lutter \MakeLowercase{\textit{et al.}}: Continuous-Time Value Iteration for Robust Policies}



\IEEEtitleabstractindextext{%
\begin{abstract}
% 1) Stating the problem: & % 2) Say why it is interesting:
Solving the Hamilton-Jacobi-Bellman equation is important in many domains including control, robotics and economics. Especially for continuous control, solving this differential equation and its extension the Hamilton-Jacobi-Isaacs equation, is important as it yields the optimal policy that achieves the maximum reward on a give task. In the case of the Hamilton-Jacobi-Isaacs equation, which includes an adversary controlling the environment and minimizing the reward, the obtained policy is also robust to perturbations of the dynamics. 
%
% 3) Say what your solution is / what it achieves:
In this paper we propose continuous fitted value iteration (cFVI) and robust fitted value iteration (rFVI). These algorithms leverage the non-linear control-affine dynamics and separable state and action reward of many continuous control problems to derive the optimal policy and optimal adversary in closed form. This analytic expression simplifies the differential equations and enables us to solve for the optimal value function using value iteration for continuous actions and states as well as the adversarial case. Notably, the resulting algorithms do not require discretization of states or actions.
%
% 4) Say what follows from your solution:
We apply the resulting algorithms to the Furuta pendulum and cartpole. We show that both algorithms obtain the optimal policy. The robustness Sim2Real experiments on the physical systems show that the policies successfully achieve the task in the real-world. When changing the masses of the pendulum, we observe that robust value iteration is more robust compared to deep reinforcement learning algorithm and the non-robust version of the algorithm. Videos of the experiments are shown at \url{https://sites.google.com/view/rfvi}
\end{abstract}

}

\maketitle

\IEEEdisplaynontitleabstractindextext

% \fi
\IEEEpeerreviewmaketitle

\begin{acronym}
\acro{hjb}[HJB]{Hamilton-Jacobi-Bellman}
\acro{hji}[HJI]{Hamilton-Jacobi-Isaacs}
\acro{sim2real}[Sim2Real]{simulation to real}
\acro{rl}[RL]{reinforcement learning}
\acro{cfvi}[cFVI]{Continuous Fitted Value Iteration}
\acro{rfvi}[rFVI]{Robust Fitted Value Iteration}
\acro{mlp}[MLP]{Multi-Layer Perceptron}
\end{acronym}

\IEEEraisesectionheading{\section{Introduction}\label{sec:introduction}}
% with the very first Section heading (almost always called "Introduction").
% \IEEEraiseSectionheading{} command. Note the need to keep any \label that
% is to refer to the Section immediately after \section in the above as
% \IEEEraiseSectionheading puts \section within a raised box.

\IEEEPARstart{O}{ne} approach to obtain the optimal control inputs that maximize the reward, is to solve the \ac{hjb} equation, as this differential equation expresses a sufficient and necessary condition for optimality~\cite{liberzon2011calculus}. Solving the \ac{hjb} yields the optimal value function, which can be used to retrieve the optimal action at each state. Therefore, this ansatz has been used by various research communities, including economics~\cite{carter2001foundations, caputo2005foundations} and robotics~\cite{lavalle2006planning, chen2015safe, bansal2017hamilton, fisac2018general}, to compute the optimal plan for a given reward function. For example in robotics, the optimal action sequence to navigate a robot to a goal from any starting state with the least actions can be obtained by solving the \ac{hjb}. Classical approaches solve this differential equation using PDE solvers on a discretized grid~\cite{chen2015safe, bansal2017hamilton, fisac2018general}. Instead of using a grid, various researchers have proposed to use the machine-learning toolset of black-box function approximation and regression techniques to solve the \ac{hjb} using randomly sampled data~\cite{doya2000reinforcement, morimoto2005robust, tassa2007least, lutter2019hjb, yang2014reinforcement, liu2014neural, kim2020L4DC, kim2020hamilton}.  

\medskip\noindent
In this article, we follow this line of research and present an approach to solve the \ac{hjb} and the adversarial extension of the \ac{hjb} using value iteration. This approach unifies the derivation of our previously proposed algorithms \ac{cfvi}~\cite{lutter2021cfvi} and \ac{rfvi}~\cite{lutter2021robust}. \ac{cfvi} is a value iteration based algorithm that solves the \ac{hjb} for continuous states and actions. This approach leverages our insights to obtain the optimal policy in closed form for control-affine dynamics and separable rewards. This analytic policy enables us to solve this differential equation using fitted value iteration with a deep network as value function approximation. Previously this would not have been possible for continuous actions as solving for the optimal action at each step would not be computationally feasible. \ac{rfvi} is similar to \ac{cfvi} but instead of solving the \ac{hjb}, this algorithms solves the \ac{hji}. The \ac{hji} incorporates an additional adversary that tries to minimize the reward. Therefore, the obtained policy and value function are robust to perturbations of the adversary. For this extension, we show that also the optimal perturbation of the adversary can be computed in closed form. Therefore, \ac{rfvi} obtains the robust optimal policy while \ac{cfvi} only obtains the optimal policy. 

\medskip\noindent
We apply the resulting algorithms to standard continuous control problems. In this setting, the \ac{hji} is of special interest as the adversary is used to control the parameters of the environment to minimize the reward~\cite{chen2015safe, bansal2017hamilton, fisac2018general}. Therefore, this min-max formulation optimizes the worst-case reward. This worst-case optimization yields an optimal policy that is robust to changes in the environment because the worst-case is assumed during planning. We show that the proposed approaches can obtain the optimal control policy and can be transferred to the physical system. By changing the masses of the dynamical systems during the \ac{sim2real} transfer experiments, we show that the \ac{rfvi} policy is more robust compared to the baselines. 

\medskip \noindent
\textbf{Summary of Contributions.} In this paper, we show that \ac{cfvi} and \ac{rfvi} obtain the optimal control policy and can be successfully transferred to a physical system. To derive these algorithms and highlight their performance,
\begin{enumerate} [wide=0pt]
\item we extend the existing derivations~\cite{lyshevski1998optimal, doya2000reinforcement, morimoto2005robust} of the optimal policy to a wider class of reward functions and adversaries,
\vspace{0.35em}
\item we propose to use value iteration to solve the differential equations as this optimization is more robust compared to the approaches using regression~\cite{doya2000reinforcement, tassa2007least, lutter2019hjb},
\vspace{0.35em}
\item we provide an extensive experimental evaluation that compares qualitative and quantitative performance evaluates the policies on the physical system using Sim2Real. Furthermore, we provide ablation studies highlighting the impact of the individual hyper-parameters. 
\end{enumerate}

\medskip \noindent
\textbf{Outline.} The article is structured as follows. First, we introduce the problem statement (Section \ref{sec:problem}). Afterwards, we derive the analytic optimal policy (Section \ref{sec:opt_policy}), extend the approach to action constraints (Section \ref{sec:action_constraints}) and derive the optimal adversary (Section \ref{sec:opt_adversary}). Section \ref{sec:value_iteration} introduces the value iteration to compute the optimal value function and the following Section describes the used value function representation (Section \ref{sec:value_function}). The experiments are summarized in Section \ref{sec:experiments}. The following conclusion (Section~\ref{sec:conclusion}) describes the observed limitations and potential future work, relates the proposed algorithm to the existing literature and summarizes the contributions of the paper.

\section{Problem Statement}\label{sec:problem}
We focus on solving the \acl{hjb} and \acl{hji} differential equation. These equations can be derived using the continuous-time \ac{rl} problem and the corresponding adversarial extension. In the following, we first introduce the continuous-time \ac{rl} problem and extend it to the adversarial formulation afterward.

\medskip \noindent \textbf{Reinforcement Learning.}
The infinite horizon continuous time reinforcement learning problem is described by
\begin{gather}
\pi^*(\vx_0) = \argmax_{\vu} \int_{0}^{\infty} \exp(-\rho t) \:\: r_{c}(\vx_t, \vu_t) \: dt,  \label{eq:cont_policy} \\ 
V^{*}(\vx_0) = \max_{\vu} \int_{0}^{\infty} \exp(-\rho t) \: r_c(\vx_{t}, \vu_{t}) \:  dt,  \label{eq:cont_val} \\
\text{with} \hspace{10pt} \vx(t) = \vx_0 + \int_{0}^{t} f_c(\vx_{\tau}, \vu_{\tau}) \: d\tau 
\end{gather}
with the discounting factor $\rho \in (0, \infty]$, the reward $r_c$ and the dynamics $f_c$ \cite{kirk2004optimal}. Notably, the discrete-time reward and discounting can be described using the continuous-time counterparts, i.e., $r(\vx, \vu) = \Delta t \: r_c(\vx, \vu)$ and $\gamma = \exp(-\rho \: \Delta t)$ with the sampling interval $\Delta t$. The continuous-time discounting $\rho$ is, in contrast to the discrete discounting factor $\gamma$, agnostic to sampling frequencies. In the continuous-time case, the Q-function does not exist \cite{doya2000reinforcement}. It is important to note that the optimization of \Eqref{eq:cont_val} is unconstrained w.r.t. to the actions. Action constraints will be implicitly introduced using the action cost in Section \ref{sec:action_constraints}.

\medskip \noindent
\Eqref{eq:cont_val} can be rewritten to yield the \ac{hjb} differential equation, which is the continuous time counterpart of the discrete Bellmann equation. Substituting the value function at time $t' = t + \Delta t$, approximating $V^{*}(x(t'), \: t')$ with its 1st order Taylor expansion and taking the limit $\Delta t \rightarrow 0$ yields the HJB described by
\begin{align}
\rho \: V^{*}(\vx) &= \max_{\vu} \:\: r(\vx, \vu) + f_c(\vx, \vu)^{\top} \: \frac{\partial V^{*}}{\partial \vx}. \label{eq:hjb} 
\end{align}
In the following, we will abbreviate $\partial V^{*}/\partial \vx$ as $\nabla_{\vx} V^{*}$. The full derivation of the HJB can be found within \cite{doya2000reinforcement}. 

% Assumptions:
\medskip \noindent 
The reward is assumed to be separable into a non-linear state reward $q_c$ and the action cost $g_c$ described by
\begin{align}
    r_c(\vx, \vu) = q_c(\vx) - g_c(\vu). \label{eq:seperable_rwd}
\end{align}
The action penalty $g_{c}$ is non-linear, positive definite, and strictly convex. This separability is common for robot control problems as rewards are composed of a state component quantifying the distance to the desired state % $\vx_{\text{des}}$ 
and an action penalty. The action cost penalizes non-zero actions to avoid bang-bang control from being optimal and is convex to have a unique optimal action.

\medskip \noindent
The deterministic continuous-time dynamics model $f_c$ is assumed to be non-linear w.r.t. the system state $\vx$ but affine w.r.t. the action~$\vu$. Such dynamics model is described by
\begin{align}
\dot{\vx} = \va(\vx; \theta) + \mB(\vx; \theta) \vu \label{eq:affine_dyn} 
\end{align}
with the non-linear drift $\va$, the non-linear control matrix $\mB$ and the system parameters $\theta$. Robot dynamics models are naturally expressed in the continuous-time formulation and many are control affine. Furthermore, this special case has received ample attention in the existing literature due to its wide applicability \cite{doya2000reinforcement, kappen2005linear, todorov2007linearly}. 

\medskip \noindent \textbf{Adversarial Reinforcement Learning.} 
The adversarial approach incorporates an adversary that controls the environment and tries to minimize the obtained reward of the policy. This formulation resembles an zero-sum two-player game, where the policy maximizes the reward and the adversary minimizes the reward. Therefore, one optimizes the worst case reward and not the expected reward. The worst-case formulation is commonly used within robust control to obtain a policy that is robust to changes in the environments. The corresponding optimization problems of \Eqref{eq:cont_policy} and \Eqref{eq:cont_val} are described by
\begin{gather}
\pi^*(\vx) = \argmax_{\pi} \: \inf_{\bm{\xi} \in \Omega} \: \int_{0}^{\infty} \exp(-\rho t) \:\: r_{c}(\vx_t, \vu_t) \: dt,
\label{eq:adv_policy} \\ 
V^{*}(\vx) = \max_{\vu} \: \inf_{\bm{\xi} \in \Omega} \: \int_{0}^{\infty} \exp(-\rho t) \: r_c(\vx_{t}, \vu_{t}) \:  dt,
\label{eq:adv_val}
% \text{with} \hspace{10pt}
% \vx(t) = \vx_0 + \int_{0}^{t} f_c(\vx_{\tau}, \vu_{\tau}, \: \bm{\xi}_{\tau}) \: d\tau, 
\end{gather}
with the adversary $\vxi$, admissible set $\Omega$. The order of the optimizations can be switched as the optimal actions and disturbance remain identical \cite{isaacs1999differential}. The adversary $\bm{\xi}$ actions are constrained to be in the set of admissible disturbances $\Omega$ as otherwise the adversary is too powerful and would prevent the policy from learning the task. Similar to the HJB, \Eqref{eq:adv_val} can be rewritten to yield the Hamilton-Jacobi-Isaacs (HJI) differential equation. The HJI is described by
\begin{align}
\rho \: V^{*}(\vx) &= \max_{\vu} \inf_{\vxi \in \Omega} \:\: r(\vx, \vu) \:\: + f_c(\vx, \vu, \vxi)^{\top} \: \nabla_{\vx} V^{*}. \label{eq:hji} 
\end{align}
The optimal policy and adversary are assumed to be stationary and Markovian. In this case, the worst-case action is deterministic if the dynamics are deterministic. If the adversary would be stochastic, the optimal policies are non-stationary and non-Markovian \cite{zhang2020robust}. This assumption is used in most of the existing literature on adversarial RL \cite{morimoto2005robust, heger1994consideration, zhang2020robust, mandlekar2017adversarially}. 

\medskip \noindent
To obtain a policy that is robust to variations of the dynamical system and bridge the simulation to reality gap, we consider four different adversaries. These adversaries either alter (1) the state~\cite{heger1994consideration, littman1994markov, nilim2005robust}, (2) action~\cite{morimoto2005robust, pinto2017robust, pinto2017supervision, tessler2019action, pattanaik2017robust, gleave2019adversarial}, (3) observation~\cite{mandlekar2017adversarially, zhang2020robust} and (4) model parameters~\cite{mandlekar2017adversarially}. Each adversary addresses a potential cause of the simulation gap. The state adversary $\vxi_{x}$ incorporates unmodeled physical phenomena in the simulation. The action adversary $\vxi_{u}$ addresses the non-ideal actuators. The observation adversary $\vxi_{o}$ introduces the non-ideal observations caused by sensors. The model adversary $\vxi_{\theta}$ introduces a bias to the system parameters. All adversaries could be subsumed via a single adversary with a large admissible set. However, the resulting dynamics would not capture the underlying structure of the simulation gap~\cite{mandlekar2017adversarially} and the optimal policy would be too conservative~\cite{xu2012robustness}. Therefore, we disambiguate between the different adversaries to capture this structure. However, all four adversaries can be combined to obtain robustness against each variation.
Mathematically, the system dynamics including the adversary are
\begin{alignat}{1}
\text{State} \hspace{5pt} \vxi_{x}: \hspace{10pt} \dot{\vx} &= \va(\vx; \vtheta) + \mB(\vx; \vtheta) \vu  + \vxi_{\vx} \label{eq:dyn_adv_state}, \\
\text{Action} \hspace{5pt} \vxi_{u}: \hspace{10pt} \dot{\vx} &= \va(\vx; \vtheta) + \mB(\vx; \vtheta) \: \left(\vu + \vxi_{\vu}\right) \label{eq:dyn_adv_action}, \\
\text{Observation} \hspace{5pt} \vxi_{o}: \hspace{10pt} \dot{\vx} &= \va(\vx + \vxi_{\vo}; \vtheta) + \mB(\vx + \vxi_{\vo}; \vtheta) \: \vu  \label{eq:dyn_adv_obs},\\ 
\text{Model} \hspace{5pt} \vxi_{\theta}: \hspace{10pt} \dot{\vx} &= \va(\vx; \vtheta + \vxi_{\theta}) + \mB(\vx; \vtheta + \vxi_{\theta}) \: \vu  \label{eq:dyn_adv_model}.
\end{alignat}
Instead of disturbing the observation, \Eqref{eq:dyn_adv_obs} disturbs the simulation state of the drift and control matrix. This disturbance is identical to changing the observed system state. 

%\medskip \noindent
%The deterministic perturbation is in contrast to the standard RL approaches, which describe $\vxi_{i}$ as a stochastic variable. However, this difference is due to the worst-case perspective, where one always selects the worst sample at each state. This worst case sample is deterministic if the policies are stationary and Markovian. The filtering approaches of state-estimation are not applicable to this problem formulation as these approaches cannot infer a state-dependent bias. % Therefore, the actions of the adversary must be incorporated into the plan of the optimal policy. 

\begin{table*}[t]
%   \tiny
\renewcommand{\arraystretch}{1.25}
\setlength\tabcolsep{10pt}
  \centering
  \caption{\footnotesize
    Selected action costs.
    The choice of the action cost $g(\vu)$
    determines the range of actions $\vu \in \textrm{dom}\,(g)$,
    as well as the form of the optimal policy $\nabla g^{*}(\vw)$
    and the type of non-linearity in the HJB equation $g^{*}(\vw)$.
    Section 1 of the table contains policies with standard action domains.
    Section 2 provides formulae for shifting and scaling actions and scaling costs.
    Section 3 \& 4 show how to use the formulaes.
%     ; $\relu(a) = \max(0, a)$;
% $L_\delta(a)$ is the Huber loss function which is quadratic for $|a| < \delta$
% and linear otherwise.
}
  \begin{tabular*}{\textwidth}{l l l l l}
    \toprule
      Policy Name
        & Action Range
        & Action Cost $g(\vu)$
        % & Gradient $\nabla g(\vu)$
        & Policy $\nabla g^{*}(\vw)$
        & HJB Nonlinear Term $g^{*}(\vw)$ \\
    \midrule
        Linear 
            & $\vu \in \mathbb{R}^{n_u}$
            & $\frac{1}{2} \vu^{\top} \mathbf{R}\vu$
            % & $2\mathbf{R} \vu$
            & $\mathbf{R}^{-1} \vw$
            & $\frac{1}{2} \vw^{\top} \mathbf{R}^{-1} \vw$ \\
        Logistic
            & $\mathbf{0} < \vu < \mathbf{1}$
            & $ \vu^{\top} \log \vu + (\mathbf{1}-\vu)^{\top}\log(\mathbf{1}-\vu)$
            % & $\log\frac{\vu}{1-\vu}$
            & $\frac{\mathbf{1}}{\mathbf{1}+e^{-\vw}} \eqcolon \sigma(\vw)$
            & $\mathbf{1}^{\top} \log\left(\mathbf{1}+e^{\vw}\right)$ \\
        Atan
            & $-\frac{\pi}{2}\mathbf{1} < \vu < \frac{\pi}{2}\mathbf{1}$
            & $-\log \cos \vu$
            % & $\tan \vu$
            & $\tan^{-1}(\vw)$
            & $\vw^{\top} \tan^{-1}(\vw)
                -\frac{1}{2} \mathbf{1}^{\top} \log (\mathbf{1} + \vw^2)$ \\
    \midrule
        Action-Scaled
            & $\vu \in \alpha \,\textrm{dom}\,(g)$
            & $\alpha g(\alpha^{-1}\vu)$
            % & $\nabla g(\alpha^{-1}\vu)$
            & $\alpha\nabla g^{*}(\vw)$
            & $\alpha g^{*}(\vw)$ \\
        Cost-Scaled
            & $\vu \in \textrm{dom}\,(g)$
            & $\beta g(\vu)$
            % & $\beta \nabla g(\vu)$
            & $\nabla g^{*}(\beta^{-1}\vw)$
            & $\beta g^{*}(\beta^{-1}\vw)$ \\
        Action-Shifted
            & $\vu \in \textrm{dom}\,(g) - \gamma \mathbf{1}$
            & $g(\vu + \gamma\mathbf{1}) - g(\gamma\mathbf{1})$
            & $\nabla g^{*}(\vw) - \gamma\mathbf{1}$
            & $g^{*}(\vw) - \gamma \mathbf{1}^{\top} \vw$ \\
    \midrule
        Tanh
            & $-\mathbf{1} < \vu < \mathbf{1}$
            & $g_{\textrm{logistic}}(\frac{\vu+\mathbf{1}}{2})
                - g_{\textrm{logistic}}(\frac{1}{2})$
            & $\tanh \vw = 2\sigma(2\vw) - \mathbf{1}$
            & $\mathbf{1}^{\top} \log\cosh \vw$ \\
        TanhActScaled
            & $-\alpha\mathbf{1} < \vu < \alpha\mathbf{1}$
            & $\alpha g_{\tanh}(\alpha^{-1}\vu)$
            & $\alpha \tanh \vw$
            & $\alpha \mathbf{1}^{\top} \log\cosh \vw$ \\
        AtanActScaled
            & $-\alpha\mathbf{1} < \vu < \alpha\mathbf{1}$
            & $-\frac{2\alpha}{\pi} \log \cos (\frac{2\alpha}{\pi} \vu)$
            & $\frac{2\alpha}{\pi} \tan^{-1}(\vw)$
            & $\frac{2\alpha}{\pi} g_{\textrm{atan}}^*(\vw)$ \\
    \midrule
        Bang-Bang
            & $-\mathbf{1} \leq \vu \leq \mathbf{1}$
            & $\chi_{[-\mathbf{1},\mathbf{1}]}(\vu),$ $\chi$ - charact. fun. %\footnote{$\xi_A(a)$ is the characteristic function from convex analysis that equals zero if $a\in A$ and $\infty$ otherwise}
            & $\sign \vw$
            & $\| \vw \|_1$ \\
        Bang-Lin
            & $-\mathbf{1} \leq \vu \leq \mathbf{1}$
            & $\frac{1}{2}\vu^{\top} \vu \; \chi_{[-\mathbf{1},\mathbf{1}]}(\vu)$
            & $-\mathbf{1} + \sum_{\delta = -1}^1 \relu(\mathbf{1} - \delta\vw)$
            & $ \mathbf{1}^{\top} L_{\mathbf{1}} (\vw),$ $L_\delta(a)$ - Huber loss  \\
    \bottomrule
  \end{tabular*}
\label{table:convex_conjugate_functions}
\end{table*}

\section{Deriving The Optimal Policy}\label{sec:opt_policy}
As a first step to solve the HJB~(\Eqref{eq:hjb}) and the HJI~(\Eqref{eq:hji}), one must obtain an efficient approach to solve the maximization w.r.t. the actions. In the case of discrete actions, this optimization can be solved by evaluating each action and choosing the action with the highest value. In the continuous action case, one cannot evaluate each action, and numerically solving an optimization problem at each state is computationally too expensive. Therefore, one requires an analytic solution to the optimization. This closed-form solution enables a computationally efficient algorithm to solve both differential equations. In the following, we show that this optimization can be solved using the previously described assumptions. 

\begin{theorem} \label{theorem:opt_policy}
If the dynamics are control affine (\Eqref{eq:affine_dyn}), the reward is separable w.r.t. to state and action (\Eqref{eq:seperable_rwd}) and the action cost $g_c$ is positive definite and strictly convex, the continuous-time optimal policy $\pi^{*}$ is described by
\begin{gather}
    \pi^{*}(\vx)  = \nabla \tilde{g}_c \left( \mB(\vx)^{\top} \nabla_{x} V^{*} \right) \label{eq:theorem}
\end{gather}
where $\tilde{g}$ is the convex conjugate of $g$ and $\nabla_{x} V^*$ is the Jacobian of current value function $V^{*}$ w.r.t. the system state.
\end{theorem} 

\begin{hproof}
The detailed proof is provided in the appendix. This derivation follows our previous work~\cite{lutter2019hjb}, which generalized the special cases initially described by Lyshevski~\cite{lyshevski1998optimal} and Doya~\cite{doya2000reinforcement} to a wider class of reward functions. Substituting \Eqref{eq:affine_dyn} and \Eqref{eq:seperable_rwd} into the HJB (\Eqref{eq:hjb}) yields
\begin{align*}
\rho V^{*}(\vx)  % &= \max_{\vu} \:\: r + \gamma V + \gamma V_x^{\top} f_c  \Delta t + \gamma \mathcal{O} \Delta t\\
                &= \max_{\vu} \:\: q_c(\vx) - g_c(\vu) + \nabla_{x}V^{\top} \left[ \va(\vx) + \mB(\vx) \vu \right].
\end{align*}
Therefore, the optimal action is described by 
\begin{align}
    \vu^{*}_{t} = \argmax_{\vu} \: \nabla_{x} V^{\top} \mB(\vx_t) \: \vu - g_{c}(\vu). \label{eq:final_opt}
\end{align}
This optimization can be solved analytically as $g_c$ is strictly convex and hence $\nabla g_{c}(\vu) = \vw$ is invertible, i.e., $\vu = \left[ \nabla g_{c}\right]^{-1}(\vw) = \nabla \tilde{g}_{c}(\vw)$ with the convex conjugate $\tilde{g}$. The solution of \Eqref{eq:final_opt} is described by
\begin{align*}
\mB^{\top} \nabla_{x}V^{*} - \nabla g_c(\vu) = 0 \hspace{5pt} \Rightarrow \hspace{5pt} \vu^{*} = \nabla \tilde{g}_c \left( \mB^{\top} \nabla_{x}V^{*} \right).
\end{align*}
\end{hproof} 
\vspace{-1.5em}

\bigskip\noindent
This closed-form optimal policy has an intuitive interpretation. The policy performs steepest ascent by following the gradient of the value function. The inner part $\mB(\vx)^{\top} \nabla_{x}V^{k}$ projects the change in state onto the action space. The action cost $g_c$ determines the magnitude of the action. The projected gradient is then reshaped using the action cost. The design of the action cost will be explained in the next section. If the continuous-time optimal policy is executed by a discrete-time controller, the time discretization determines the step-size of the hill climbing. Therefore, the time discretization affects the convergence to the value function maximum. If the step size is too large the system becomes unstable and does not converge to the maximum. For most real-world robotic systems with natural frequencies below $5$Hz and control frequencies above $100$Hz, the resulting step-size is sufficiently small to achieve convergence to the value function maximum. Therefore, $\pi^{*}$ can be used for high-frequency discrete-time controllers. Furthermore, the continuous-time policy can be used for intermittent control (event-based control) where interacting with the system occurs at irregular time-steps and each interaction is associated with a cost \cite{astrom2008event}.  

\subsection{Action Constraints} \label{sec:action_constraints}
The previous Section derived the analytic expression that describes the impact of the action cost on the optimal actions. Therefore, the action cost can be used to design the shape of the policy. By selecting a specific shape of the policy, one can leverage the convex conjugacy to derive the corresponding action cost.
The shape of the optimal policy is determined by the monotone function $\nabla g^*$. Therefore, one can define any desired monotone shape and determine the corresponding strictly convex cost by inverting $\nabla g^*$ to compute $\nabla g$ and integrating $\nabla g$ to obtain the strictly convex cost function $g(\vu)$. This approach can be used to design action costs such that the common standard controllers become optimal. For example, bang-bang control is optimal w.r.t. to no action cost. The linear policy is optimal w.r.t. the quadratic action cost. The logistic policy is optimal w.r.t. the binary cross-entropy cost. The Atan shaped policy is optimal w.r.t. the log cosine cost. Incorporating the action constraints directly via barrier-shaped action cost is beneficial as clipping the unbounded actions is only optimal for linear systems \cite{de2000elucidation} and increasing the quadratic action cost to ensure the action limits leads to over-conservative behavior and underuse of the control range. Furthermore, this implicit integration of the action limits via the action cost enables to only solve the unconstrained optimization problem rather than incorporating the action constraints via an explicit constraint. The full generality of this concept based on convex conjugacy is shown in Table~\ref{table:convex_conjugate_functions}, which shows the corresponding cost functions for linear, logistic, atan, tanh, and bang-bang controllers.

\medskip \noindent
Using the rules from convex analysis~\cite{boyd2004convex},
the action limits and action range can be adapted
as shown by Action-Scaled, Action-Shifted,
and Cost-Scaled rows in Table~\ref{table:convex_conjugate_functions}.
This enables quick experimentation by mixing and matching costs.
For example, the action cost corresponding to the tanh policy
is straightforwardly derived using the well-known relationship
between $\tanh(x)$ and the logistic sigmoid $\sigma(x)$
given by $\tanh(x) = 2\sigma(2x) - 1$. 
Note that a formula for general invertible affine transformations can be derived,
not only for scalar scaling and shifting. 
Classical types of hard nonlinearities~\cite{ching2010quasilinear}
can be derived as limiting cases of smooth solutions.
For example, taking the Tanh action cost $g_{\tanh}$
and scaling it with $\beta \to 0$, i.e., putting a very small cost on actions,
results in the Bang-Bang control shape.
Taking a different limit of the Tanh policy in which scaling is performed
simultaneously w.r.t. the action and cost,
the resulting shape is what we call Bang-Lin and corresponds to a function
which is linear around zero and saturates for larger input values.

\subsection{Optimal Adversary Actions} \label{sec:opt_adversary}
\begin{table*}[!t]
\centering
\renewcommand{\arraystretch}{1.7}
\setlength\tabcolsep{7pt}
\caption{% \footnotesize
    The optimal actions $\vu^k$ and adversarial actions $\xi^{k}$ for the state-, action-, model- and observation bias with the admissible set $\Omega$. \vspace{-10pt}
}
\begin{tabular*}{\textwidth}{l c c c c}
    \toprule
    & State Perturbation
    & Action Perturbation
    & Model Perturbation
    & Observation Perturbation \\
    \midrule
    Dynamics $f_c(\vx, \vu, \xi)$
    & $\va(\vx) + \mB(\vx) \vu  + \xi$
    & $\va(\vx) + \mB(\vx) (\vu  + \xi)$
    & $\va(\theta + \xi) + \mB(\theta + \xi) \vu$
    & $\va(\vx + \xi) + \mB(\vx + \xi) \vu$ \\
    % \midrule
    Optimal Action $\vu^{k}$
    & $\nabla \tilde{g}(\mB(\vx)^{\top} \nabla_x V^{k})$
    & $\nabla \tilde{g}(\mB(\vx)^{\top} \nabla_x V^{k})$
    & $\nabla \tilde{g}(\mB(\vx)^{\top} \nabla_x V^{k})$
    & $\nabla \tilde{g}(\mB(\vx)^{\top} \nabla_x V^{k})$ \\
    % \midrule
    Optimal Disturbance $\xi^{k}$
    & $-h_{\Omega} \left( \nabla_x V^{k} \right)$
    & $-h_{\Omega} \left( \mB^{\top} \nabla_x V^{k}  \right)$
    & $-h_{\Omega} \left( \left(\frac{\partial \va}{\partial \vx} + \frac{\partial \mB}{\partial \theta} \vu^{k} \right)^{\top} \nabla_x V^{k}  \right)$
    & $-h_{\Omega} \left( \left(\frac{\partial \va}{\partial \vx} + \frac{\partial \mB}{\partial \vx} \vu^{k} \right)^{\top} \nabla_x V^{k} \right)$ \\
\bottomrule
\end{tabular*}
\label{table:adversarial_disturbances}
\end{table*}

\noindent 
To solve the HJI efficiently one not only requires the optimal policy to be described using an analytic form but also the optimal adversary. To obtain this solution one must solve the constrained min-max optimization of \Eqref{eq:hji}. We show that this optimization problem can be solved analytically for the described dynamics and disturbance models using the Karush–Kuhn–Tucker conditions. It is important to note that the adversaries are not exclusive and can be combined. We only derive the individual cases for simplicity. 

\medskip\noindent
The resulting optimal actions $\vu^{*}$ and disturbances $\vxi^{*}_{i}$ have a coherent intuitive interpretation. The optimal actions perform steepest ascent by following the gradient of the value function~$\nabla_x V$. The optimal perturbations perform steepest descent by following the negative gradient of the value function. The magnitude of taken action is determined by the action cost $g$ in the case of the optimal policy or the admissible set $\Omega$ in the case of the adversary. The optimal policy and the optimal adversary is described by 
\begin{align}
\vu^{*} &= \nabla \tilde{g} \left(\frac{\partial f_c(.)}{\partial \vu}^{\top} \nabla_x V^{*} \right),  &
\vxi^{*}_i &= - h_{\Omega} \left(\frac{\partial f_c(.)}{\partial \vxi_{i}}^{\top} \nabla_x V^{*}\right). \label{eq:opt_adv}
\end{align}
% The gradients are transformed to the controlled quantity $\vy$ using the Jacobian transpose of $\dot{\vx}$, i.e., $\partial \dot{\vx} / \partial \vy$. For the considered cases, the Jacobians are described by
% \begin{align*}
% \frac{\partial \dot{\vx}}{\partial \vxi_{x}} &= \mI,  &
% \frac{\partial \dot{\vx}}{\partial \vu} &= \frac{\partial \dot{\vx}}{\partial \vxi_{u}} = \mB, &
% \frac{\partial \dot{\vx}}{\partial \vxi_{o}} &= \frac{\partial f_c}{\partial \vx},  & % (\vx, \vu)
% \frac{\partial \dot{\vx}}{\partial \vxi_{m}} &= \frac{\partial f_c}{\partial \vtheta} 
% \end{align*}
% with the identity matrix $\mI$.
In the following we abbreviate $\left[\partial f_c(.) / \partial \vy\right]^{\top} \nabla_{x} V$, as $\vz_{y}$.
%In the case of quadratic cost, i.e., $\textstyle \frac{1}{2}\: \vu^{\top} \mR \vu$, $\vz_{u}$ is linearly rescaled with $\mR^{-1}$. 
For the adversarial policy, $h_{\Omega}$ rescales $\vz_{\xi}$ to be on the boundary of the admissible set. If the admissible set bounds the signal energy to be smaller than $\alpha$, the disturbance is rescaled to have the length $\alpha$. Therefore, the adversary is described by
\begin{align}
\Omega_{E} = \{ \vxi \in \mathbb{R}^{n} \: | \:  \| \vxi \|_{2} \leq \alpha \} \hspace{5pt} \Rightarrow \hspace{5pt} h_{E}(\vz_{\xi}) = \alpha \: \frac{\vz_{\xi}}{\| \vz_{\xi} \|_{2}}. \label{eq:signal_energy}    
\end{align}
If the amplitude of the disturbance is bounded, the disturbance performs bang-bang control. In this case the adversarial policy is described by
\begin{equation}
\begin{aligned}
\Omega_{A} &= \{ \vxi \in \mathbb{R}^{n} \: | \:  \bm{\nu}_{\text{min}} \leq \vxi \leq \bm{\nu}_{\text{max}} \} \\
&\hspace{65pt}\Rightarrow\hspace{5pt} h_{A}(\vz_{\xi}) = \vDelta \sign\left( \vz_{\xi} \right) + \vmu\label{eq:amplitude},
\end{aligned}
\end{equation}
with $\vmu = \left( \vnu_{\text{max}} + \vnu_{\text{min}} \right) / 2$ and $\vDelta = \left( \vnu_{\text{max}} - \vnu_{\text{min}} \right) / 2$.  

\medskip
\noindent
The following theorem derive Equations \ref{eq:opt_adv}, \ref{eq:signal_energy} and \ref{eq:amplitude} for the optimal policy and the different disturbances. Following the theorem, we provide sketches of the proofs for the state and model disturbance. The remaining proofs are analogous. The complete proofs for all theorems are provided in the appendix. All solutions are summarized in Table \ref{table:adversarial_disturbances}.

\begin{theorem}
If the dynamics are control affine (\Eqref{eq:affine_dyn}), the reward is separable w.r.t. to state and action (\Eqref{eq:seperable_rwd}) and the action cost $g_c$ is positive definite and strictly convex, the continuous-time optimal policy $\pi^{*}$ and optimal adversary $\vxi^{*}$ can be computed in closed form.  

\bigskip \noindent
\textbf{2.1 State Disturbance.} The optimal policy $\pi^{*}$ and state disturbance $\vxi_{x}$ (\Eqref{eq:dyn_adv_state}) with bounded signal energy (\Eqref{eq:signal_energy}) is described by
\begin{align*}
\pi^{*}(\vx) &= \nabla \tilde{g} \left(\mB(\vx)^{\top} \nabla_x V^{*}\right), & \vxi^{*}_{x} &= - \alpha \frac{\nabla_x V^{*}}{\| \nabla_x V^{*} \|_2}.
\end{align*}

\bigskip \noindent
\textbf{2.2 Action Disturbance.} The optimal policy $\pi^{*}$ and action disturbance $\vxi_{u}$ (\Eqref{eq:dyn_adv_action}) with bounded signal energy (\Eqref{eq:signal_energy}) is described by
\begin{align*}
\pi^{*}(\vx) &= \nabla \tilde{g} \left(\mB(\vx)^{\top} \nabla_x V^{*}\right), & \vxi_{u} &= - \alpha \frac{\mB(\vx)^{\top} \nabla_x V^{*}}{\| \mB(\vx)^{\top} \nabla_x V^{*} \|_2}.
\end{align*}

\bigskip \noindent
\textbf{2.3 Observation Disturbance.} The optimal policy $\pi^{*}$ and observation disturbance $\vxi_{\theta}$ (\Eqref{eq:dyn_adv_obs}) with bounded signal energy (\Eqref{eq:signal_energy}), smooth drift and control matrix (i.e., $\va, \mB \in C^{1}$) and $\mB(\vx + \vxi_{o}) \approx \mB(\vx)$ is described by
\begin{gather*}
\pi(\vx) = \nabla \tilde{g} \left(\mB(\vx)^{\top} \nabla_x V\right), \hspace{30pt}
\vxi_{o} = - \alpha \frac{\vz_{o}}{\| \vz_{o} \|_2} \\
\text{with} \hspace{5pt} \vz_{o} = \left( \frac{\partial \va(\vx; \: \vtheta)}{\partial \vx} + \frac{\partial \mB(\vx; \: \vtheta)}{\partial \vx} \pi(\vx) \right)^{\top} \nabla_x V.
\end{gather*}

\bigskip \noindent
\textbf{2.4 Model Disturbance.} The optimal policy $\pi^{*}$ and model disturbance $\vxi_{\theta}$ (\Eqref{eq:dyn_adv_model}) with element-wise bounded amplitude (\Eqref{eq:amplitude}), smooth drift and control matrix (i.e., $\va, \mB \in C^{1}$) and $\mB(\vtheta + \vxi_{\theta}) \approx \mB(\vtheta)$ is described by
\begin{gather*}
\pi(\vx) = \nabla \tilde{g} \left(\mB(\vx)^{\top} \nabla_x V\right), \hspace{20pt}
\vxi_{\theta} = -\vDelta_{\nu} \sign \left( \vz_{\vtheta }\right) + \vmu_{\nu} \\
\text{with} \hspace{5pt} \vz_{\theta} = \left(\frac{\partial \va(\vx;\: \vtheta)}{\partial \vtheta} + \frac{\partial \mB(\vx; \: \vtheta)}{\partial \vtheta} \pi(\vx) \right)^{\top} \nabla_x V,
\end{gather*}
the parameter mean $\vmu_{\vnu} = \left( \vnu_{\text{max}} + \vnu_{\text{min}} \right) / 2$ and parameter range $\vDelta_{\vnu} = \left( \vnu_{\text{max}} - \vnu_{\text{min}} \right) / 2$.
\end{theorem}

\medskip \noindent 
\textbf{Proof Sketch Theorem 2.1} For the admissible set $\Omega_{E}$, \Eqref{eq:hji} can be written with the explicit constraint. 
This optimization is described by
\begin{align*}
\rho V^{*} &= \max_{\vu} \: \min_{\vxi_x} \: r(\vx, \vu) + f(\vx, \vu, \vxi_x)^{\top} \nabla_{x} V^{*}
\hspace{5pt} 
\text{s.t.} \hspace{5pt}  \vxi_{x}^{\top} \vxi_{x} \leq \alpha^{2}, \\ 
&= \max_{\vu} \: \min_{\vxi_x} \: q_c(\vx) - g_c(\vu) + \left[\va(\vx) + \mB(\vx) \vu  + \vxi_{\vx} \right]^{\top} \nabla_{x} V^{*}. 
\end{align*}
% Substituting the Taylor expansion for $V(\vx_{t+1})$, the dynamics model and the reward, the optimization is described by
% \begin{align*}
% V_{\text{tar}} &= \max_{\vu} \: \min_{\vxi_x} \: r + \gamma V + \gamma \nabla_x V^{\top} f_c \Delta t + \gamma \mathcal{O}(\vx, \vu, \Delta t) \Delta t %\\
% % \frac{V_\text{tar} - \gamma V}{\Delta t} &= q_c + \max_{\vu} \left[\gamma \nabla_x V^{\top} \left(\va + \mB \vu \right) + \gamma \mathcal{O}(.) - g_c \right] + \min_{\vxi} \left[\nabla_x V^{\top} \vxi_x \right]
% \end{align*}
% with the higher order terms $\mathcal{O}(\vx, \vu, \Delta t)$. In the continuous-time limit, the higher-order terms and the discounting disappear, i.e., $\lim_{\Delta t\rightarrow 0} \mathcal{O}(\vx, \vu, \Delta t) {=} 0$ and $\lim_{\Delta t\rightarrow 0} \exp(-\rho \Delta t) {=} 1$. 
Therefore, the optimal action is described by %\lim_{\Delta t\rightarrow 0} \gamma = 
\begin{align*}
\vu_{t} = \argmax_{\vu} \:  \nabla_x V^{\top} \mB \: \vu - g_{c}(\vu) \hspace{5pt} \Rightarrow \hspace{5pt}
\vu_{t} = \nabla \tilde{g}_c\big(\mB^{\top} \nabla_x V\big).
\end{align*}
The optimal state disturbance is described by
\begin{align*}
\vxi^{*}_x = \argmin_{\vxi_x} \: \nabla_x V^{\top} \vxi_x \hspace{10pt} \text{s.t.} \hspace{10pt} \frac{1}{2} \left[ \vxi_{x}^{\top} \vxi_{x} - \alpha^{2} \right] \leq 0.
\end{align*}
This constrained optimization can be solved using the Karush-Kuhn-Tucker (KKT) conditions. The resulting optimal adversarial state perturbation is described by
\begin{align*}
\vxi_{x} = - \alpha \frac{\nabla_x V}{\| \nabla_x V \|_2}.
\end{align*}
\hspace{\fill}\qed

\medskip \noindent
\textbf{Proof Sketch Theorem 2.4} \Eqref{eq:hji} can be written as
\begin{align*}
\rho V^{*} &= \max_{\vu} \: \min_{\vxi_x} \: r(\vx, \vu) + f(.)^{\top} \nabla_{x} V^{*}
\hspace{5pt} \text{s.t.} \hspace{5pt} \left(\vxi_{\theta} - \vmu_{\nu} \right)^{2} \leq  \vDelta_{\nu}^{2}
\end{align*}
by replacing the admissible set $\Omega_{A}$ with an explicit constraint. In the following we abbreviate $\mB(\vx; \vtheta + \vxi_{\theta})$ as $\mB_{\xi}$ and $\va(\vx; \vtheta + \vxi_{\theta})$ as $\va_{\xi}$. 
Substituting \Eqref{eq:seperable_rwd} and \Eqref{eq:dyn_adv_model} simplifies the optimization to
\begin{align*}
\vu^{*}, \vxi^{*}_{\theta} = \argmax_{\vu} \argmin_{\vxi} \left[\big(\va_{\xi} + \mB_{\xi} \vu \big)^{\top} \nabla_x V^{*} - g_c(\vu) \right].
\end{align*}
This nested max-min optimization can be solved by first solving the inner optimization w.r.t. to $\vu$ and substituting this solution into the outer maximization. The Lagrangian for the optimal model disturbance is described by
\begin{align*}
\vxi^{*} = \argmin_{\vxi} \: \left(\va_{\xi} + \mB_{\xi} \vu \right)^{\top} \nabla_x V^{*} + \frac{1}{2} \vlambda^{\top} \left((\vxi_{\theta} - \vmu_{\nu})^{2} - \vDelta_{\nu}^{2}\right).
\end{align*}
Using the KKT conditions this optimization can be solved. The stationarity condition yields
\begin{gather*}
\vz_{\theta} + \vlambda^{\top} (\vxi_{\theta} - \vmu_{\nu}) = 0  \hspace{15pt} \Rightarrow \hspace{15pt} \vxi^{*}_{\theta} = - \vz_{\theta} \oslash \vlambda + \vmu_{\nu} % \\
% \text{with} \hspace{5pt} \vz_{\theta} =  \left[ \frac{\partial \va}{\partial \vtheta} + \frac{\partial \mB}{\partial \vtheta} \vu \right]^{\top} \nabla_x V
\end{gather*}
with the elementwise division $\oslash$. Using the primal feasibility and the complementary slackness, the optimal $\vlambda^{*}$ can be computed.
The resulting optimal model disturbance is described by 
\begin{align*}
\vxi^{*}_{\theta}(\vu) = % -\vDelta \circ \vz_{\theta} \oslash \| \vz_{\theta}\|_1 +  \vmu_{\nu} =
-\vDelta_{\nu} \sign\big( \vz_{\theta}(\vu) \big) + \vmu_{\nu}
\end{align*}
as $\vz_{\theta} \oslash \| \vz_{\theta}\|_1 = \sign(\vz_{\theta})$. The action can be computed by 
\begin{align*}
\vu^{*} = \argmax_{\vu} \nabla_x V^{\top} \left[\va(\vxi^{*}_{\theta}(\vu)) + \mB\left(\vxi^{*}_{\theta}(\vu)\right) \vu \right] - g_c(\vu).
\end{align*}
Due to the envelope theorem~\cite{carter2001foundations}, the extrema is described by
\begin{align*}
\mB(\vx; \vtheta + \vxi^{*}_{\theta}(\vu))^{\top} \nabla_x V - g_{c}(\vu) = 0.
\end{align*}
This expression cannot be solved without approximation as $\mB$ does not necessarily be invertible w.r.t. $\vtheta$. Approximating $\mB(\vx; \vtheta + \vxi^{*}(\vu)) \approx \mB(\vx; \vtheta)$, lets one solve for $\vu$. In this case the optimal action $\vu^{*}$ is described by  
$\vu^{*} {=} \nabla \tilde{g}(\mB(\vx; \vtheta)^{\top} \nabla_x V)$. This approximation implies that neither agent or the adversary can react to the action of the other and must choose simultaneously. This assumption is common in prior works~\cite{bansal2017hamilton}. 
\hspace{\fill}\qed

\section{Continuous Fitted Value Iteration} \label{sec:value_iteration}
The previous sections showed that the optimizations contained within the \ac{hjb} and \ac{hji} can be solved in closed form. Leveraging these insights, we will derive the proposed algorithms to solve the \ac{hjb} and \ac{hji} using fitted value iteration.

\subsection{Algorithm}
Substituting the optimal actions and adversary, simplifies the \Eqref{eq:hjb} and \Eqref{eq:hji} to a differential equations without optimization. The differential equations are described by
\begin{gather*}
\rho \: V^{*}(\vx) = r(\vx, \vu^{*}) \:\: + f_c(\vx, \vu^{*})^{\top} \: \nabla_{\vx} V^{*}, \\ 
\rho \: V^{*}(\vx) = r(\vx, \vu^{*}) \:\: + f_c(\vx, \vu^{*}, \vxi^{*})^{\top} \: \nabla_{\vx} V^{*}, %\\
%\text{with} \hspace{5pt} \vu^{*} = \nabla \tilde{g} \left(\frac{\partial f_c(.)}{\partial \vu}^{\top} \nabla_x V^{*} \right),  \hspace{5pt} 
%\vxi^{*}_i = - h_{\Omega} \left(\frac{\partial f_c(.)}{\partial \vxi_{i}}^{\top} \nabla_x V^{*}\right).
\end{gather*}
with the optimal action $\vu^{*}$ and optimal adversary $\vxi^{*}$ described by \Eqref{eq:opt_adv}. Within the machine learning community various approaches have been proposed to solve these differential equations using standard regression techniques \cite{doya2000reinforcement, morimoto2005robust, tassa2007least}, including our previous work \cite{lutter2019hjb}. The problem with these approaches is, that these differential equation do not have a unique solutions without considering the boundary constraint, which implies that the optimal action always prevent the system to leave the state domain. The boundary condition is described by
\begin{align} \label{eq:bnd}
  f(\bar{\vx}, \vu^{*})^{\top} \bm{\eta}(\bar{\vx}) \leq 0 \hspace{15pt} \text{for} \hspace{5pt} \bar{\vx} \in \partial \mathcal{X}
\end{align}
with the outward pointing normal vector $\bm{\eta}$ defined on the state domain boundary $\partial \mathcal{X}$ makes the solution unique \cite{fleming2006controlled}. In the case of LQR, this boundary condition implies the positive-definiteness of the quadratic value function. Incorporating this boundary constraint within the optimization problem is challenging, as the state domain boundary is unknown and the commonly used black-box function approximators are local, hence incoporating the boundary constraint within the optimization does not ensure a globally coherent solution. To overcome this shortcoming the existing regression approaches use additional optimization tricks and specific value function representations. 

\medskip \noindent \textbf{Continuous Value Iteration.}
To overcome the problems with solving the differential equation using regression, we use fitted value iteration~(FVI)~\cite{boyan1994generalization, baird1995residual, tsitsiklis1996feature, munos2008finite}. FVI is an extension of the classical dynamic programming Value iteration~(VI) \cite{bellman1957dynamic} to continuous states using a function approximator. FVI iteratively computes the value function target and minimizes the $\ell_p$-norm between the target and the approximation $V^{k}(\vx; \:\psi)$ until the value function has converged. Mathematically, this approach is described by
\begin{gather}
V_{\text{tar}}(\vx_t) = \max_{\vu} \: r(\vx_{t}, \vu) + \gamma V^{k}(\vx_{t+1}; \: \psi_k), \label{eq:fvi_update} \\ 
\psi_{k+1} = \argmin_{\psi} \sum_{\vx \in \mathcal{D}} \| V_{\text{tar}}(\vx) - V^{k}(\vx; \:\psi)  \|_{p}^{p} \label{eq:fitting}
\end{gather}
with the parameters $\psi_k$ at iteration $k$ and the fixed dataset~$\mathcal{D}$. While for discrete states and actions for $\gamma < 1$, VI is proven to converge to the optimal value function~\cite{puterman1994markov}, this convergence proof of VI does not generalize to FVI as the fitting of the value target is not necessarily a contraction~\cite{boyan1994generalization, baird1995residual, tsitsiklis1996feature, munos2008finite}. However, empirically this approach has been successfully used to retrieve the optimal value function or Q function \cite{boyan1994generalization, baird1995residual, tsitsiklis1996feature, munos2008finite, tesauro1992practical, ernst2005tree, massoud2009regularized, riedmiller2005neural, mnih2015human}. To solve the HJB and HJI, the value function target must be adapted and is described by
\begin{gather*}
    V_{\text{tar}}(\vx_t) = r\left(\vx_{t}, \vu^{*}\right) + \gamma V^{k}(\vx_{t+1}; \: \psi_k), \\
    \text{with} \hspace{10pt} \vx_{t+1} = \vx_t + \int_{t}^{t+\Delta T} f_c\left(\vx_{\tau}, \vu_{\tau}^{*}, \vxi_{\tau}^{*} \right) d\tau,
\end{gather*}
and the time step $\Delta T$. The selection of the time step is important as this discretization affects the convergence speed which is proportional to $\gamma$. As $\Delta t$ decreases, $\gamma$ increases, i.e., $\gamma = \lim_{\Delta t \rightarrow 0} \exp(-\rho \Delta t) = 1$. Therefore, the contraction coefficient of VI decreases exponentially with increasing sampling frequencies. This slower convergence is intuitive as higher sample frequencies effectively increase the number of steps to reach the goal. 

\medskip \noindent \textbf{$\mathbf{N}$-Step Value Function Target.}
To further improve the convergence speed of fitted value iteration, the exponentially weighted n-step value function target is used. This value target is described by
\begin{gather*}
V_{\text{tar}}(\vx) = \int_{0}^{\top} \beta \: \exp(-\beta t) \: R_t \: dt + \exp(-\beta T) R_T,  \\ 
% \text{with} \hspace{5pt}
R_t = \int_{0}^{t} \exp(-\rho \tau) \: r_c(\vx_{\tau}, \vu_{\tau}) d\tau + \exp(-\rho t) V^{k}(\vx_t), \label{eq:R_t}
\end{gather*}
and the exponential decay constant $\beta$, can be used. This approach is the continuous-time counterpart of the discrete eligibility trace of TD($\lambda$) with $\lambda = \exp(- \beta \Delta t)$ \cite{sutton1998introduction}. With respect to deep RL, this discounted n-step value target is similar to the generalized advantage estimation (GAE) of PPO \cite{schulman2015high, schulman2017proximal} and model-based value expansion (MVE) \cite{feinberg2018model, buckman2018sample}. GAE and MVE have shown that the $n$-step target increases the sample efficiency and lead to faster convergence to the optimal policy. The integrals can be solved using any ordinary differential equation solver with fixed or adaptive step-size. We use the explicit Euler integrator with fixed steps to solve the integral for all samples in parallel using batched operations on the GPU. The nested integrals can be computed efficiently by recursively splitting the integral and reusing the estimate of the previous step. In practice we treat $\beta$ as a hyperparameter and select $T$ such that the weight of the $R_T$ is $\exp\left(-\beta T\right) = 10^{-4}$.

\medskip \noindent
\textbf{Dataset.}
\Eqref{eq:fitting} fits the value function using the dataset~$\mathcal{D}$. This dataset can be fixed as in dynamic programming or offline/batch RL or a replay memory containing the visited states of the current policy $\pi^{k}$. We refer to the latter as real-time dynamic programming (RTDP) as Barto et. al.~\cite{barto1995learning} introduced the online version of dynamic programming first. In the case of the fixed dataset, the dataset can either originate from a previous learning process, which is frequently used in the offline RL benchmarks \cite{gulcehre2020rl}, or uniformly sampled from the state domain $\mathcal{X}$. Within this work, we sample uniformly from the state domain as the used state dimensionality is low-dimensional. For RTDP, the dataset is a replay memory containing the visited states of the current policy as in most modern deep reinforcement learning algorithms. In this case, the exploration of the policy is important as the policy needs to cover the state space to discover high reward configurations. In the offline case, no exploration is needed. 

\medskip \noindent
\textbf{Admissible Set.}
For the state, action, and observation adversary the signal energy is bounded. We limit the energy of $\vxi_{x}$, $\vxi_{u}$ and $\vxi_{o}$ as the non-adversarial disturbances are commonly modeled as multivariate Gaussian distribution. Therefore, the average energy is determined by the noise covariance matrix. For the model parameters $\vtheta$ a common practice is to assume that the approximate model parameters have a model error of up to~$\pm15\%$ \cite{muratore2018domain, muratore2021data}. Hence, we bound the amplitude of each component. To not overfit to the deterministic worst-case system of $V$ and enable the discovery of good actions, the amplitude of the adversarial actions of $\xi_x$, $\xi_u$, $\xi_o$ is modulated using a Wiener process. This random process allows a continuous-time formulation that is agnostic to the sampling frequency.

\medskip\noindent \textbf{Algorithms.}
Combining value iteration with the analytic optimal policy and adversary yields the two algorithms \ac{cfvi}~\cite{lutter2021value} and \ac{rfvi}~\cite{lutter2021robust}. While \ac{cfvi} is used to solve the HJB, \ac{rfvi} is used to solve the HJI. We refer to these algorithms as an extension of value iterations as these algorithms also extend FVI approach to continuous actions and adversarial RL, which was previously not possible. Previously FVI was limited to discrete actions. Furthermore, we differentiate between a dynamic programming version of the algorithm using a fixed dataset, e.g., DP \ac{cfvi} and DP \ac{rfvi} and the online version using a replay memory, i.e., RTDP \ac{cfvi} and RTDP \ac{rfvi}. The algorithms are summarized in algorithm \ref{alg:rFVI}. 

\begin{algorithm}[t]
\caption{Robust Fitted Value Iteration (rFVI)}
\label{alg:rFVI}
\begin{algorithmic}
\STATE {\bfseries Input:} Model $f_{c}(\vx, \vu)$, Dataset $\mathcal{D}$ \& Admissible Set $\Omega_{\xi}$
\STATE {\bfseries Result:} Value Function $V^{*}(\vx;\: \psi^{*})$
\WHILE{not converged}
\STATE // Compute Value Target for $\vx \in \mathcal{D}$:\;
\STATE $\vx_{\tau} = \vx_i + \int_{0}^{\tau} f_c(\vx_{t}, \vu_{t}, \vxi^{x}_{t}, \vxi^{u}_{t}, \vxi^{o}_{t}, \vxi^{\theta}_{t}) dt$
\STATE $R_t = \int_{0}^{t} \exp(-\rho \tau) \: r_c(\vx_{\tau}, \vu_{\tau}) d\tau + \exp(-\rho t) V^{k}(\vx_t)$
\STATE $V_{\text{tar}}(\vx_i) = \int_{0}^{\top} \beta \: \exp(-\beta t) \: R_t \: dt + \exp(-\beta T) R_T$
% \STATE $\vu_{\tau} = \nabla \tilde{g}\left(\mB(\vx_{\tau}) \nabla_x V^{k}(\vx_{\tau}))\right)$  
% \STATE $\vxi^{x}_{\tau} = -h_{\Omega}\left(\nabla_x V^{k}(\vx_{\tau})\right)$
% \STATE $\vxi^{u}_{\tau} = -h_{\Omega}\left(\mB(\vx_{\tau})^{\top} \nabla_x V^{k}(\vx_{\tau})\right)$
% \STATE $\vxi^{o}_{\tau} = -h_{\Omega}\left(\left[\frac{\partial \va(\vx; \: \theta)}{\partial \vx} + \frac{\partial \mB(\vx; \: \theta)}{\partial \vx} \vu_{\tau} \right]^{\top} \nabla_x V^{k}(\vx_{\tau}) \right)$
% \STATE $\vxi^{\theta}_{\tau} = -h_{\Omega}\left(\left[\frac{\partial \va(\vx; \: \theta)}{\partial \theta} + \frac{\partial \mB(\vx; \: \theta)}{\partial \theta} \vu_{\tau} \right]^{\top} \nabla_x V^{k}(\vx_{\tau})\right)$
% \STATE $\vxi^{o}_{\tau} = -h_{\Omega}\left(\left[\frac{\partial \va}{\partial \vx} + \frac{\partial \mB}{\partial \vx} \vu_{\tau} \right]^{\top} \nabla_x V^{k}(\vx_{\tau}) \right)$
% \STATE $\vxi^{\theta}_{\tau} = -h_{\Omega}\left(\left[\frac{\partial \va}{\partial \vtheta} + \frac{\partial \mB}{\partial \vtheta} \vu_{\tau} \right]^{\top} \nabla_x V^{k}(\vx_{\tau})\right)$
\STATE
\STATE // Fit Value Function:
\STATE $\psi_{k+1} = \argmin_{\psi} \sum_{\vx \in \mathcal{D}} \| V_{\text{tar}}(\vx) - V(\vx; \psi) \|^{p}$ \;
\STATE
\IF{RTDP rFVI}
\STATE // Add samples from $\pi^{k+1}$ to FIFO buffer $\mathcal{D}$
\STATE $\mathcal{D}^{k+1} = h(\mathcal{D}^{k}, \{\vx^{k+1}_0 \: \dots \: \vx^{k+1}_N \})$
\ENDIF
\ENDWHILE
\end{algorithmic}
\end{algorithm}

\begin{table*}[t]
\tiny
% \scriptsize
\centering
\renewcommand{\arraystretch}{1.1}
\caption{\footnotesize
Average rewards on the simulated and physical systems. The average ranking describes the decrease in reward compared to the best result averaged on all systems. Therefore, a small decrease shows that the algorithm performs close to the best algorithm on each system.
The initial state distribution during training is noted by~$\mu$. The dynamics are either deterministic model $\theta \sim \delta(\theta)$ or sampled using uniform domain randomization $\theta \sim \mathcal{U}(\theta)$.
During evaluation the roll outs start with the pendulum pointing downwards.\vspace{-10pt}
}
\setlength{\tabcolsep}{7.pt}
\begin{tabular*}{\textwidth}{l c c c c c c c c | c c  c c | c}
% \begin{tabular*}{\textwidth}{l r r | c c c | c c c}
\toprule
 & & & \multicolumn{2}{c}{\textbf{Simulated Pendulum}}  & \multicolumn{2}{c}{\textbf{Simulated Cartpole}} & \multicolumn{2}{c|}{\textbf{Simulated Furuta Pendulum}}   & \multicolumn{2}{c}{\textbf{Physical Cartpole}} & \multicolumn{2}{c|}{\textbf{Physical Furuta Pendulum}} & \textbf{Average} \\
%  \cmidrule(lr){4-6} \cmidrule(lr){7-9}
& & & % \multicolumn{1}{c}{Algorithm} & $\mu$ & $\vtheta$ &  
Success & Reward & Success & Reward & Success & Reward  & Success & Reward & Success & Reward  & \textbf{Ranking}
\\  
 % \multicolumn{1}{c}{Algorithm} 
 \multicolumn{1}{c}{Algorithm} & $\mu$ & $\theta$ &   [$\%$] & [$\mu \pm 2 \sigma$] & [$\%$] & [$\mu \pm 2 \sigma$] & [$\%$] & [$\mu \pm 2 \sigma$] & [$\%$] & [$\mu \pm 2 \sigma$] & [$\%$] & [$\mu \pm 2 \sigma$] & [$\%$] \\
 \cmidrule(lr){1-3} \cmidrule(lr){4-5} \cmidrule(lr){6-7} \cmidrule(lr){8-9} \cmidrule(lr){10-11} \cmidrule(lr){12-13} \cmidrule(lr){14-14} 
DP rFVI (ours) & $-$ & $\delta(\theta)$ 
& $100.0$ & $-032.7 \pm 000.3$ 
& $100.0$ & $-027.1 \pm 004.8$ 
& $100.0$ & $-041.3 \pm 010.8$ 
& $100.0$ &\textcolor{black}{$\mathbf{-074.1 \pm 040.3}$}
& $100.0$ & $-278.0 \pm 034.3$ 
& \textcolor{black}{$-062.7$}
\\
DP cFVI (ours) & $-$ & $\delta(\theta)$ 
& $100.0$ &\textcolor{black}{$\mathbf{-030.5 \pm 000.8}$}
& $100.0$ &\textcolor{black}{$\mathbf{-024.2 \pm 002.1}$}
& $100.0$ &\textcolor{black}{$\mathbf{-027.7 \pm 001.6}$}
& $73.3$ & $-143.7 \pm 210.4$ 
& $100.0$ &\textcolor{black}{$\mathbf{-082.1 \pm 007.6}$}
& \textcolor{black}{$-019.2$}
\\
RTDP cFVI (ours) & $\mathcal{U}$ & $\delta(\theta)$ 
& $100.0$ &\textcolor{black}{$\mathbf{-031.1 \pm 001.4}$}
& $100.0$ &\textcolor{black}{$\mathbf{-024.9 \pm 001.6}$}
& $100.0$ & $-040.1 \pm 002.7$ 
& $100.0$ &\textcolor{black}{$-101.1 \pm 029.0$}
& $00.0$ & $-1009.9 \pm 004.5$ 
& \textcolor{black}{$-247.7$}
\\
\cmidrule(lr){1-3} \cmidrule(lr){4-5} \cmidrule(lr){6-7} \cmidrule(lr){8-9} \cmidrule(lr){10-11} \cmidrule(lr){12-13} \cmidrule(lr){14-14}
SAC & $\mathcal{N}$ & $\mathcal{U}(\theta)$ 
& $100.0$ &\textcolor{black}{$\mathbf{-031.1 \pm 000.1}$}
& $100.0$ & $-026.9 \pm 003.2$ 
& $100.0$ & $-029.3 \pm 001.5$ 
& $00.0$ & $-518.6 \pm 028.1$ 
& $86.7$ & $-330.7 \pm 799.0$ 
& \textcolor{black}{$-185.8$}
\\
SAC \& UDR & $\mathcal{N}$ & $\delta(\theta))$ 
& $100.0$ & $-032.9 \pm 000.6$ 
& $100.0$ & $-029.7 \pm 004.6$ 
& $100.0$ & $-032.0 \pm 001.1$ 
& $100.0$ & $-394.8 \pm 382.8$ 
& $100.0$ & $-181.4 \pm 157.9$ 
& \textcolor{black}{$-120.8$}
\\
SAC & $\mathcal{U}$ & $\mathcal{U}(\theta)$ 
& $100.0$ &\textcolor{black}{$\mathbf{-030.6 \pm 001.4}$}
& $100.0$ &\textcolor{black}{$\mathbf{-024.2 \pm 001.4}$}
& $100.0$ &\textcolor{black}{$\mathbf{-028.1 \pm 002.0}$}
& $53.3$ & $-144.5 \pm 204.0$ 
& $100.0$ & $-350.8 \pm 433.3$ 
& \textcolor{black}{$-086.5$}
\\
SAC \& UDR & $\mathcal{U}$ & $\mathcal{U}(\theta)$ 
& $100.0$ & $-031.4 \pm 002.5$ 
& $100.0$ &\textcolor{black}{$\mathbf{-024.2 \pm 001.3}$}
& $100.0$ &\textcolor{black}{$\mathbf{-028.1 \pm 001.3}$}
& $40.0$ & $-296.4 \pm 418.9$ 
& $100.0$ & $-092.3 \pm 064.1$ 
& \textcolor{black}{$-063.8$}
\\
\cmidrule(lr){1-3} \cmidrule(lr){4-5} \cmidrule(lr){6-7} \cmidrule(lr){8-9} \cmidrule(lr){10-11} \cmidrule(lr){12-13} \cmidrule(lr){14-14}
DDPG & $\mathcal{N}$ & $\mathcal{U}(\theta)$ 
& $100.0$ &\textcolor{black}{$\mathbf{-031.1 \pm 000.4}$}
& $98.0$ & $-050.4 \pm 285.6$ 
& $100.0$ & $-030.5 \pm 003.5$ 
& $06.7$ & $-536.7 \pm 262.7$ 
& $46.7$ & $-614.1 \pm 597.8$ 
& \textcolor{black}{$-281.4$}
\\
DDPG \& UDR & $\mathcal{N}$ & $\delta(\theta))$ 
& $100.0$ & $-032.5 \pm 000.5$ 
& $100.0$ & $-027.4 \pm 002.3$ 
& $100.0$ & $-034.6 \pm 009.8$ 
& $00.0$ & $-517.9 \pm 117.6$ 
& $86.7$ & $-192.7 \pm 404.8$ 
& \textcolor{black}{$-156.6$}
\\
DDPG & $\mathcal{U}$ & $\mathcal{U}(\theta)$ 
& $100.0$ & $-031.5 \pm 000.7$ 
& $100.0$ & $-028.2 \pm 005.5$ 
& $100.0$ & $-030.0 \pm 001.7$ 
& $06.7$ & $-459.4 \pm 248.3$ 
& $100.0$ & $-146.6 \pm 218.3$ 
& \textcolor{black}{$-126.0$}
\\
DDPG \& UDR & $\mathcal{U}$ & $\mathcal{U}(\theta)$ 
& $100.0$ & $-032.5 \pm 003.6$ 
& $100.0$ & $-027.2 \pm 001.0$ 
& $100.0$ & $-032.1 \pm 001.5$ 
& $00.0$ & $-318.1 \pm 063.4$ 
& $100.0$ & $-156.7 \pm 246.4$ 
& \textcolor{black}{$-091.7$}
\\
\cmidrule(lr){1-3} \cmidrule(lr){4-5} \cmidrule(lr){6-7} \cmidrule(lr){8-9} \cmidrule(lr){10-11} \cmidrule(lr){12-13} \cmidrule(lr){14-14}
PPO & $\mathcal{N}$ & $\mathcal{U}(\theta)$ 
& $100.0$ & $-032.0 \pm 000.2$ 
& $100.0$ & $-031.5 \pm 007.2$ 
& $100.0$ & $-081.1 \pm 018.3$ 
& $00.0$ & $-287.9 \pm 068.8$ 
& $33.3$ & $-718.7 \pm 456.1$ 
& \textcolor{black}{$-261.7$}
\\
PPO \& UDR & $\mathcal{N}$ & $\delta(\theta))$ 
& $100.0$ & $-032.3 \pm 000.6$ 
& $100.0$ & $-084.0 \pm 007.8$ 
& $100.0$ & $-040.9 \pm 004.6$ 
& $00.0$ & $-435.4 \pm 111.9$ 
& $46.7$ & $-935.7 \pm 711.6$ 
& \textcolor{black}{$-370.0$}
\\
PPO & $\mathcal{U}$ & $\mathcal{U}(\theta)$ 
& $100.0$ & $-033.4 \pm 004.7$ 
& $99.0$ & $-039.7 \pm 045.7$ 
& $100.0$ & $-038.2 \pm 013.1$ 
& $00.0$ & $-183.8 \pm 018.0$ 
& $60.0$ & $-755.3 \pm 811.0$ 
& \textcolor{black}{$-219.4$}
\\
PPO \& UDR & $\mathcal{U}$ & $\mathcal{U}(\theta)$ 
& $100.0$ & $-035.6 \pm 003.1$ 
& $100.0$ & $-044.8 \pm 021.4$ 
& $100.0$ & $-048.5 \pm 006.2$ 
& $40.0$ & $-143.8 \pm 016.1$ 
& $100.0$ & \textcolor{black}{$\mathbf{-080.6 \pm 010.8}$}
& \textcolor{black}{$-054.4$}
\\
\bottomrule
\end{tabular*}
\vspace{-2.5em}
\label{table:cFVI_results}
\end{table*}

\subsection{Value Function Representation} \label{sec:value_function}
For the value function representation one can use any differentiable black-box function approximator. One common choice is a feed-forward network, i.e., \ac{mlp}, as this approximator enables an efficient computation of the value function gradient w.r.t. the network inputs. %

\medskip \noindent \textbf{Network Architecture.}
While the standard network architectures are sufficient, one can improve the performance by leveraging insights from the common control cost choices to structure the architecture. These structured representations are preferable as these limit the hypothesis space of the representable value functions. For continuous control tasks, the state reward is often a negative distance measure between $\vx_t$ and the desired state $\vx_{\text{des}}$. Hence, $q_c$ is negative definite, i.e., $q(\vx) < 0 \:\: \forall \:\: \vx \neq \vx_{\text{des}}$ and $q(\vx_{\text{des}}) = 0$. 
These properties imply that $V^{*}$ is a negative Lyapunov function, as $V^{*}$ is negative definite, $V^{*}(\vx_{\text{des}}) = 0$ and $\nabla_{x}V^{*}(\vx_{\text{des}}) = \mathbf{0}$ \cite{khalil2002nonlinear}. With a deep network a similar representation can be achieved by
\begin{align*}
    % V(\vx; \: \psi) &= f(\vx; \:\psi) - f(\vx_{\text{tar}}; \:\psi) - \frac{\partial f(\vx_{\text{tar}}; \:\psi)}{\partial \vx} \vx \\
    % V_x(\vx; \: \psi) &= \frac{\partial f(\vx; \:\psi)}{\partial \vx} - \frac{\partial f(\vx_{\text{tar}}; \:\psi)}{\partial \vx}
    V(\vx; \: \psi) &= -\left(\vx -  \vx_{\text{des}}\right)^{\top} \mL(\vx;\:\psi) \mL(\vx;\:\psi)^{\top} \left(\vx -  \vx_{\text{des}}\right) % \\
\end{align*}
with $\mL$ being a lower triangular matrix with positive diagonal. This positive diagonal ensures that $\mL \mL^{\top}$ is positive definite. Simply applying a ReLu activation to the last layer of a deep network is not sufficient as this would also zero the actions for the positive values and $\nabla_{x}V^{*}(\vx_{\text{des}}) = \mathbf{0}$ cannot be guaranteed. The local quadratic representation guarantees that the gradient and hence, the action, is zero at the desired state. However, this representation can also not guarantee that the value function has only a single extrema at~$\vx_{\text{des}}$ as required by the Lyapunov theory. In practice, the local regularization of the quadratic structure to avoid high curvature approximations is sufficient as the global structure is defined by the value function target. $\mL$ is the mean of a deep network ensemble with $N$ independent parameters $\psi_i$. The ensemble mean smoothes the initial value function and is differentiable. Similar representations have been used by prior works in the safe reinforcement learning community~\cite{gu2016continuous, berkenkamp2017safe, richards2018lyapunov, kolter2019learning, chang2019neural,bharadhwaj2021csc}. It is important to point out that this network architecture is different from NAF~\cite{gu2016continuous} as NAF uses a Q-function that is quadratic w.r.t. the actions while we use a value function that is quadratic w.r.t. to the state. 

\medskip \noindent \textbf{Gradient Projection of State Transformations.}
Additional state transformations can be incorporated into the value function to enable easier representations. For example, the standard feature transform for a continuous revolute joint maps the joint state $\vx = [ \theta,\: \dot{\theta}]$ to $\vz = [\sin(\theta), \: \cos(\theta), \: \dot{\theta} ]$ can be incorporated to avoid the discontinuity at $\pm \pi$. In this case the transformed state $\vz$ lies on the tube shaped manifold. Therefore, the value function gradient must be projected into the tangent space, which is not guaranteed when using deep networks. For the state transformation $h(\vx)$ with $V(x;\: \psi) = f(h(\vx); \: \psi)$ this projection is described by $\nabla_{x}V(x;\: \psi) = \partial f(h(\vx); \: \psi)/ \partial h \:\: \partial h(\vx)/\partial \vx$ and the gradient points in a sensible direction.

\section{Experiments} \label{sec:experiments}
In the non-linear control experiments, we apply \ac{cfvi} and \ac{rfvi} to control under-actuated systems. The sim2real experiments test the policy robustness by transferring the learned policy to the physical system and compare their performance to the standard deep \ac{rl} approaches. More precisely, we want to answer the following questions:

\medskip\noindent
\textbf{Q1:} Can \ac{cfvi} \& \ac{rfvi} obtain the optimal policies that control the simulated system?

\medskip\noindent
\textbf{Q2:} What are the qualitative differences between the policies obtained by \ac{cfvi} \& \ac{rfvi}?

\medskip\noindent
\textbf{Q3:} Does the $n$-step value target improved the convergence speed of the optimal policy?

\medskip\noindent
\textbf{Q4:} How does the admissible set of adversaries affect the performance of the optimal policy?

\medskip\noindent
\textbf{Q5:} Is the locally quadratic value function architecture beneficial compared to a standard feed-forward network?

\medskip\noindent
\textbf{Q6:} Are the obtained policies robust enough to be transferred to the real system with varying physical parameters?

\subsection{Experimental Setup}
To answer these research questions, we apply the proposed algorithms to non-linear sim2real control of under-actuated systems and compare the performance to standard actor-critic deep RL approaches. 

\medskip \noindent 
\textbf{Systems.} The physical cartpole and Furuta pendulum are manufactured by Quanser \cite{quanser} and voltage controlled. For the approximate simulation model, we use the rigid-body dynamics model with the parameters supplied by the manufacturer. If we add negative weights to the pendulum, we attach the weights to the opposite lever of the pendulum. This moves the center of mass of the pendulum closer to the rotary axis. Therefore, this shift reduces the downward force and is equivalent to a lower pendulum mass.

\medskip
\noindent
\textbf{Baselines.} The performance is compared to the actor-critic deep RL methods: DDPG \cite{lillicrap2015continuous}, SAC \cite{haarnoja2018soft} and PPO \cite{schulman2017proximal}. The robustness evaluation is only performed for the best performing baselines on the nominal physical system. % The best baselines for the Furuta pendulum are PPO-U \& UDR, SAC-U \&  UDR, and DDPG-U \& UDR and for the cartpole SAC-U, SAC-U \& UDR and PPO-U \& UDR. 
The initial state distribution is abbreviated by \{SAC, PPO, DDPG\}-U for a uniform distribution of the pendulum angle and \{SAC, PPO, DDPG\}-N for a Gaussian distribution. The baselines with Gaussian initial state distribution did not achieve robust performance on the nominal system. If the baseline uses uniform domain randomization the acronym is appended with UDR. For each of the baselines, the optimal time step is determined using a hyperparameter sweep. 

\medskip
\noindent
\textbf{Evaluation.}
To evaluate rFVI and the baselines we separately compare the state and action reward as these algorithms optimize a different objective. Hence, these algorithms trade-off state and action associated rewards differently. It is expected that the worst-case optimization uses higher actions to prevent deviation from the optimal trajectory. On the physical system, the performance is evaluated using the $25$th, $50$th, and $75$th~reward percentile as the reward distribution is multi-modal. 

\begin{figure*}[t]
    \centering
    \includegraphics[width=\textwidth]{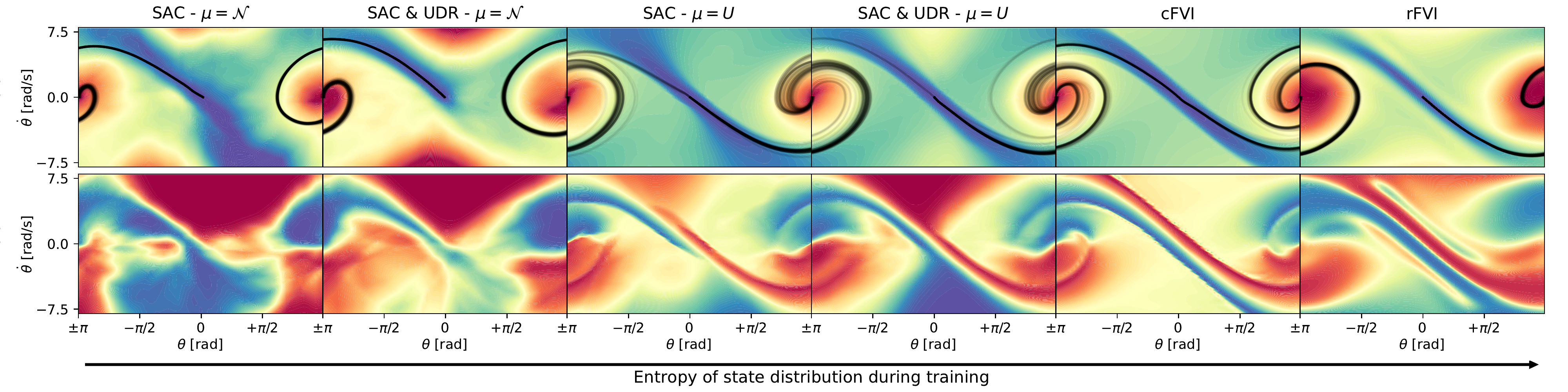}
    \vspace{-2.em}
    \caption{The optimal Value function $V^{*}$ and policy $\pi^{*}$ of rFVI, cFVI, and four different variations of SAC. All policies achieve nearly identical reward on the nominal dynamics model. The variations of SAC demonstrate the change of the policy when increasing the entropy of the state distribution during training. The entropy is increased by enlarging the initial state distribution $\mu$ and using domain randomization. For SAC and $\mu = \mathcal{N}(\pm \pi, \sigma)$ the optimal policy is only valid on the optimal trajectory. For SAC UDR and $\mu = \mathcal{U}(-\pi, +pi)$, the policy is applicable on the complete state domain. rFVI and cFVI perform value iteration on the compact state domain and naturally obtain an optimal policy applicable on the complete state-domain. rFVI adapts $V^{*}$ and $\pi^{*}$ to have a smaller ridge leading up to the upright pendulum and exerts higher actions when deviating from the optimal trajectory.}
    \label{fig:value_fun_pendulum}
    \vspace{-1.em}
\end{figure*} 

\begin{figure*}[t]
    \centering
    \includegraphics[width=\textwidth]{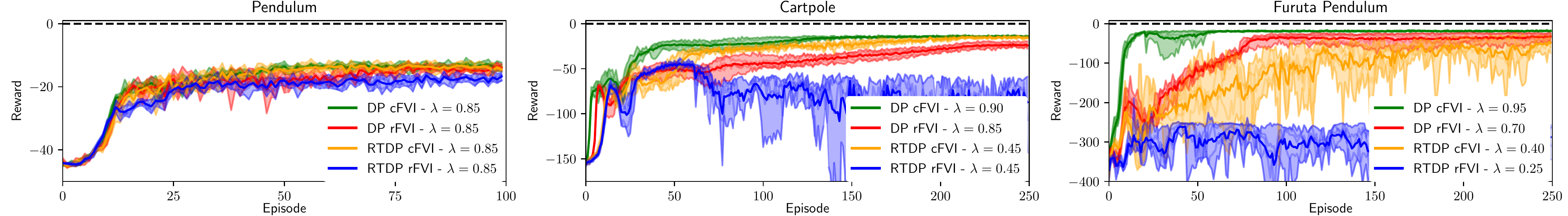}
    \vspace{-2.em}
    \caption{The learning curves for DP rFVI, DP cFVI, RTDP cFVI, and RTDP rFVI averaged over $5$ seeds. The shaded area displays the \emph{min/max} range between seeds. DP rFVI learns slower compared to DP cFVI on the carpole and Furuta pendulum as the adversary prevents learning. RTDP rFVI does not learn the task as the adversary is too strong for the online variant of rFVI despite using the identical admissible set as the offline variant DP rFVI.} 
    \label{fig:learning_curves}\vspace{-0.8em}
\end{figure*} 

\subsection{Experimental Results}
For each of the research questions we summarize the empirical results and answer the question in the respective section. Videos of all performed experiments are available at \url{https://sites.google.com/view/rfvi}

\medskip\noindent
\textbf{Q1 Control Performance.}
The learning curves for the three dynamical systems are shown in Figure~\ref{fig:learning_curves}. The quantitative comparison to the baselines is summarized in Table~\ref{table:cFVI_results}. DP~\ac{cfvi}, DP~\ac{rfvi} and RTDP~\ac{cfvi} obtain a policy that performs the swing-up and balances the pendulum. Only RTDP \ac{rfvi} does not obtain a successful policy for the cartpole and the Furuta pendulum (See Q4 for additional details). DP~\ac{cfvi} learns the fastest compared to the other variants. RTDP~\ac{cfvi} learns slower due to the required exploration while for DP~\ac{rfvi} the adversary slows down the convergence to the optimal value function. Quantitatively, the DP~\ac{cfvi} performs comparably to the best performing deep \ac{rl} algorithms. DP~\ac{rfvi} obtains a lower reward compared to DP~\ac{cfvi} due to the adversary. This reward difference is expected \ac{rfvi} minimizes the risk and hence, selects a more conservative solution with lower reward.

\medskip \noindent
\textbf{Q2 Policy Difference.}
The main difference between the policies obtained by \ac{cfvi} and \ac{rfvi} is that the robust variant converges to a stiffer policy. This stiffer policy exerts higher actions as soon as the system state leaves the optimal trajectory. This behavior is caused due to the adversary, which frequently perturbs the system state to leave the optimal trajectory. This difference can be visualized for the pendulum (Figure \ref{fig:value_fun_pendulum}). For the \ac{cfvi} policy the color gradient is much smoother, while for the \ac{rfvi} policy the color gradient changes abrupt between the maximum actions. Therefore, the \ac{rfvi} policy performs close to bang-bang control. Furthermore, the ridge leading up to the balancing point in the center is much smaller for \ac{rfvi} as the policy expects the adversary to push the state off the cliff, leading to a much lower reward. Therefore, the \ac{rfvi} policy is more conservative and uses a larger safety margin.

\medskip \noindent
\textbf{Q3 $\mathbf{N}$-Step Value Target.}
\begin{figure*}[t]
    \centering
    \includegraphics[width=\textwidth]{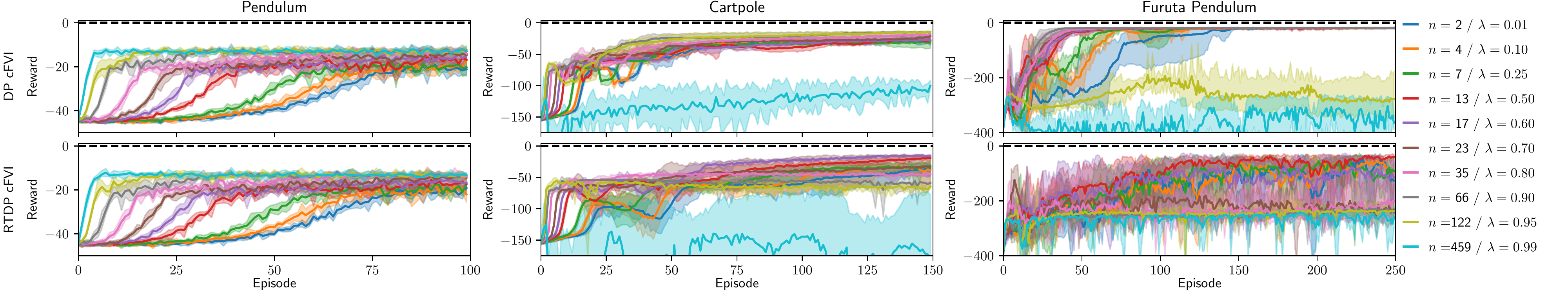}
    \vspace{-2.em}
    \caption{The learning curves averaged over $5$ seeds for the $n$-step value function target. 
    %for different exponential discounting constants $\lambda = \exp(-\beta \Delta t)$. 
    The shaded area displays the \emph{min/max} range between seeds. The step count is selected such that~$\lambda^{n} = 10^{-4}$. Increasing the horizon of the value function target increases the convergence rate to the optimal value function. For very long horizons the learning diverges as it over fits to the current value function approximation. Furthermore, the performance of the optimal policy also increases with roll out length.
    }
    \label{fig:ablation_lambda} \vspace{-0.8em}
\end{figure*} 
\begin{figure*}[h]
    \centering
    \includegraphics[width=\textwidth]{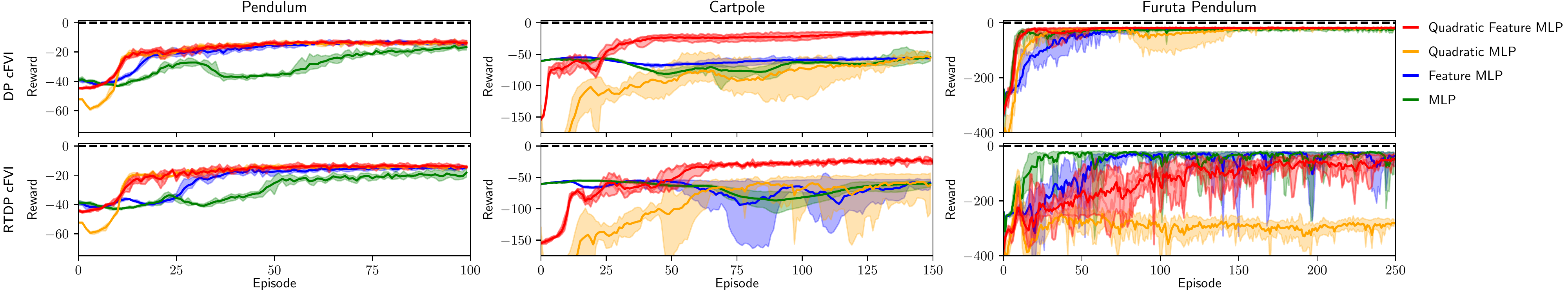}
    \vspace{-2.5em}
    \caption{The learning curves averaged over $5$ seeds for the different model architectures. The shaded area displays the \emph{min/max} range between seeds. All network architectures are capable of learning the value function and policy for most of the tasks. The locally quadratic network architecture increases learning speed compared to the baselines. The structured architecture acts as an inductive bias that shapes the exploration. The global maximum of the locally quadratic value function is guaranteed at $\vx_{\text{des}}$ and hence the initial policy performs hill-climbing towards this point.
    }
    \label{fig:ablation_architecture}
\end{figure*} 
The learning curves for varying $n$-steps is shown in Figure~\ref{fig:ablation_lambda}. The ablation study shows that the convergence speed increases when increasing the number of steps. It is important to point out that the learning speed increases w.r.t. to the number of episodes. In terms of computational cost, the number of steps increases the computational complexity by $n$ as the simulation is sequential. A surprising observation is that the cartpole and Furuta pendulum do not converge for very long trajectories despite using the true model. This degraded performance is due to overfitting to the value function approximation during training. As the value function is randomly initialized at the beginning, using long trajectories leading to potentially untrained regions in the state domain might lead to bad local optima. For the RTDP variant, this effect is amplified as the value function approximation is limited to the current state distribution and long trajectories are prone to leaving the state distribution. Therefore, the optimal step count is lower for RTDP than for DP. Furthermore, we also observed that \ac{rfvi} performs better when using slightly lower horizons compared to \ac{cfvi}. 

\medskip \noindent
\textbf{Q4 Admissible Set.} 
The learning curves for the linearly scaled admissible set are shown in figure \ref{fig:ablation_adversary}. One can observe that making the adversary more powerful, i.e., increasing $\alpha$, slightly decreases the obtained reward as the policy obtained by DP~\ac{rfvi} is more conservative. Furthermore, the learning speed is decreased. However, for all $\alpha$ the policy learns to complete the task. For RTDP~\ac{rfvi} increasing the admissible set of the adversary has a more significant impact. If the adversary becomes too powerful, the policy does not achieve the task. For example, for the cartpole the reward initially increases but drops for large admissible sets. For the Furuta pendulum the reward does not even increase at the beginning for larger $\alpha$. This behavior is due to the limited exploration of the policy. The policy does not discover that an action sequence exists to achieve the task despite the adversary. Therefore, the policy often converges to a pessimistic solution that does not exert any action. This policy is locally optimal as the adversary will always prevent the policy from completing the task. Hence, performing actions and incurring action penalties is not desirable. This effect is more pronounced for the Furuta pendulum as the system is more sensitive due to its small lengths and masses.

\medskip \noindent
\textbf{Q5 Network Architecture.} 
The learning curves for different network architectures are shown in Figure~\ref{fig:ablation_architecture}. The learning curves show that the locally quadratic network architecture using the sine/cosine transform for continuous revolute joints performs the most reliable. While the standard \ac{mlp} with feature transform performs well for the pendulum and Furuta pendulum, it does not obtain the optimal policy for the cartpole.

\medskip \noindent
\textbf{Q6 Sim2Real Transfer.} 
The results of the sim2real transfer are summarized in Figure~\ref{fig:sim2real_transfer} and Table~\ref{table:cFVI_results}. And extensive video documentation showing the performance on the physical systems is provided at \textcolor{blue}{\url{https://sites.google.com/view/rfvi}}. Both \ac{cfvi} and \ac{rfvi} can be transferred to the physical system and achieve a successful swing-up. On the nominal Furuta pendulum \ac{cfvi} obtains a higher reward than \ac{rfvi}, as \ac{rfvi} uses higher actions. On the nominal cartpole, \ac{rfvi} has a higher reward and higher success rate compared to \ac{cfvi}. In terms of robustness w.r.t. changes of the physical parameters, \ac{rfvi} achieves a reliable swing-up even when weights are added to the pendulum. Therefore, the obtained state reward is not affected by the varied mass (Figure \ref{fig:sim2real_transfer}). This difference can be nicely observed in figure \ref{fig:img_furuta}. The pendulum trajectory of the \ac{rfvi} policy is identical for all weight configurations. In contrast, the \ac{cfvi} policy needs multiple unsuccessful tries until the pendulum is upright and balanced for added weights $\geq 3$g. However, during balancing the Furuta pendulum, \ac{rfvi} performs bang-bang control, which leads to chattering due to minor delays in the control loop. In contrast to \ac{rfvi}, \ac{cfvi} keeps the pendulum still when balancing as the policy does not apply so strong actions. For the cartpole, \ac{rfvi} obtains a robust policies that performs a more consistent swing-up and balancing (Figure~\ref{fig:sim2real_transfer}). The difference between the policies is especially visible during the balancing of the cartpole. Due to the stiff \ac{rfvi} policy, the large actions immediately break the stiction of the linear actuator. Therefore, the cart is balanced at the center. For the \ac{cfvi} policy the cart oscillates around the center as the pendulum needs to fall a bit until the deviation is large enough to exert actions that break the stiction.

\medskip\noindent 
When comparing the performance of \ac{cfvi} and \ac{rfvi} to the deep \ac{rl} algorithms with uniform domain randomization, the proposed algorithms perform comparable or better. For example on the nominal system, \ac{cfvi} performs as good as the best deep \ac{rl} baseline. For example, on the Furuta pendulum most baselines complete the task but \ac{cfvi} obtains the highest reward and only PPO with domain randomization obtains a similar reward. On the robustness experiments, \ac{rfvi} performs better than the deep \ac{rl} baselines with domain randomization. The baselines including domain randomization start to fail when additional weights are added to the Furuta pendulum and the cartpole. 
\begin{figure*}[t]
    \centering
    \includegraphics[width=\textwidth]{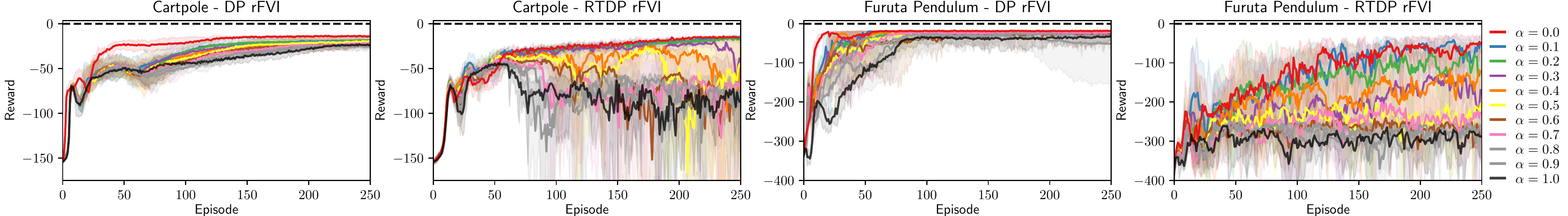}
    \vspace{-2.em}
    \caption{The learning curves for DP rFVI and RTDP rFVI with different adversary amplitudes averaged over $5$ seeds. The shaded area displays the \emph{min/max} range between seeds. The $\alpha$ corresponds to the percentage of the admissible set for all adversaries, i.e., with increasing $\alpha$ the adversary becomes more powerful. For DP rFVI the stronger adversaries do affect the final performance only marginally. For RTDP rFVI the adversaries become too powerful for small $\alpha$ and prevent learning of the optimal policy. This effect is especially distinct for the Furuta pendulum as this system is very sensitive due to the low masses. Therefore, DP rFVI can learn a good optimal policy despite very strong adversaries.  
    } 
    \label{fig:ablation_adversary}\vspace{-0.8em}
\end{figure*} 

\begin{figure*}[t]
    \centering
    \includegraphics[width=\textwidth]{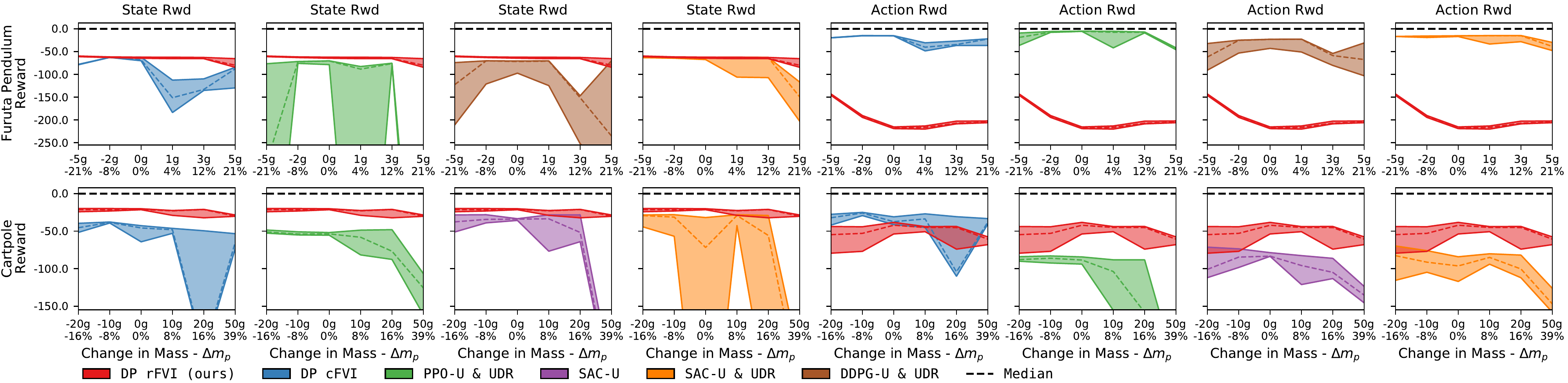}
    \vspace{-2.em}
    \caption{The $25$th, $50$th, and $75$th reward percentile for the physical Furuta pendulum and cartpole with varied pendulum weights. DP rFVI achieves a higher state reward for real-world systems compared to the baselines. For the different weights, the reward remains nearly constant. For the Furuta pendulum the action cost is significantly higher compared to the baselines as the DP rFVI causes a chattering during balancing due to the high actions and minor time delays in the control loop. If only the swing-up phase is considered the rewards are comparable.}
    \label{fig:sim2real_transfer}\vspace{-0.8em}
\end{figure*} 

\section{Conclusion}\label{sec:conclusion}
In this section, we will first discuss the surprising experimental observations, limitations and future extensions of the proposed algorithms. Afterwards, we embed our contributions within the existing literature by relating our methods to the related work. Finally we summarize the contributions of this paper. 
\subsection{Discussion}
To obtain optimal and robust policies, we optimized the worst case reward rather than the expected reward, used dynamic programming on the complete state domain rather than local exploration and assumed a known dynamics model. While these assumptions enabled us to learn good policies for the sim2real transfer, these assumptions also have several drawbacks. We want to discuss the consequences of these assumptions and propose extensions to alleviate these limitations in future work. 

\medskip \noindent
\textbf{Worst Case Optimization.} The experiments showed that the worst-case optimization increases the policy robustness. However, the policy stiffness can also cause new problems. For example, the high stiffness of the policy makes the policy more susceptible to small control loops delays leading to chattering as observed on the Furuta Pendulum. Therefore, the worst-case optimization is a double-edged sword that depending on the system might be beneficial or cause additional problems. In addition, the admissible set must be manually tuned to yield not overly conservative/pessimistic policies. One approach to overcome this limitation is to learn the magnitude of the admissible using data. In this case, one would interleave the offline planning with the online evaluation on the physical system. In every iteration, one would use the obtained real-world data to update the admissible set. Until one approaches a policy that can solve the task but is not overly conservative. Within the domain randomization community, this automatic tuning of the perturbed parameters has become widely used and improved performance. Furthermore, one could parametrize the admissible set to be state-dependent to obtain higher robustness only when needed.

\medskip \noindent
\textbf{State Distribution \& Dimensionality.} The experiments showed that the state distribution significantly affects the policy robustness and performance. For the sim2real transfer, only the baselines with a uniform initial state distribution achieved a successful transfer to the physical system~(Table \ref{table:cFVI_results}). Furthermore, the dynamic programming variants, sampling uniformly from the complete state domain, performed much better than the real-time dynamic programming variants. This increased policy robustness is intuitive as the dynamic programming mitigates the distribution shift between simulation and the real-world system. However, the dynamic programming approaches cannot scale to high dimensional systems as sampling the complete state domain becomes unfeasible for such systems. Therefore, an interesting future research question is how to obtain a sufficiently large state distribution such that the policy is robust when transferred to the physical system. This question is different from the traditional exploration exploitation trade-off as this question focuses more on minimizing the distribution mismatch.

\medskip \noindent
\textbf{Exploration.} To improve the RTDP variants and scale to higher dimensional systems, the exploration of the proposed algorithms must be improved as the algorithms sometimes do not discover the optimal solution. This problem is especially pronounced for \ac{rfvi} as in this case the adversary prevents discovering the optimal solution and the policy converges to a pessimistic policy (Figure \ref{fig:ablation_adversary}). The dynamic programming variants are not affected by this as this approach does not require exploration. The main problem for the exploration is the high-frequency sampling of the exploration noise required to solve the integrals. In this case, the exploration noise averages out and does not lead to diverse exploration. One approach to solve this would be to use model-predictive control for exploration. In this case, one would optimize the action sequence online and use the actions of the optimal policy only as prior.

\begin{figure*}[t]
    \centering
    \includegraphics[width=\textwidth]{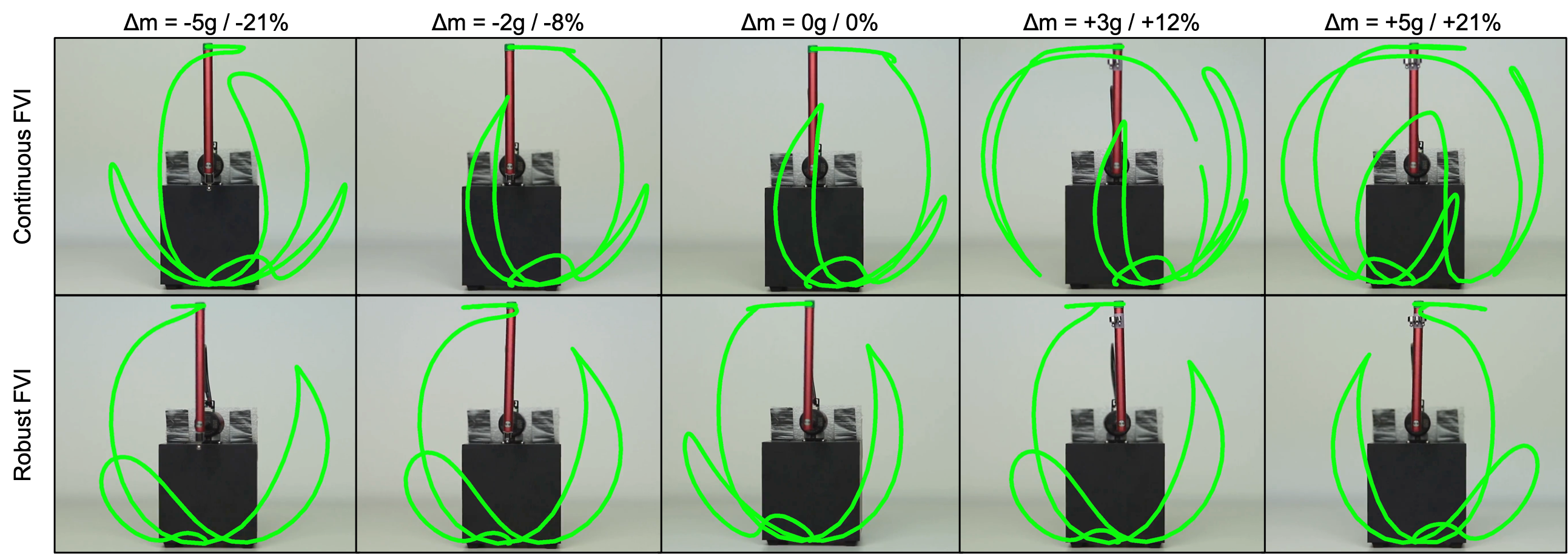}
    \vspace{-2.em}
    \caption{The tracked trajectories for DP rFVI and DP cFVI on the Furuta pendulum for different pendulum weights. The trajectories of rFVI do not significantly change when the pendulum mass is altered. For DP cFVI the trajectories start to change when additional weight is added. For these system dynamics, DP cFVI requires some failed swing-ups until the policy can balance the pendulum.}
    \label{fig:img_furuta}\vspace{-0.8em}
\end{figure*}  

\medskip \noindent
\textbf{Known Dynamics Model.} All performed experiments used the analytic equations of motion provided by the manufacturer as the model. Therefore, we assumed that the model is known. The proposed policy optimization can be combined with model learning to learn the model from data. For example, a continuous-time control-affine system dynamics can be learned using Deep Lagrangian Networks~\cite{lutter2018deep, Lutter2019Energy} or the SymODEN extension Hamiltonian Neural Network extension~\cite{zhong2019symplectic}. In future work, one should combine continuous-time policy optimization with continuous-time model learning.

\subsection{Related Work} \label{sec:related_work}
In the related work section we embed the proposed algorithms within the existing literature. We summarize the existing work on continuous-time RL, value iteration and robust policy optimization and highlight the differences of our approach compared to prior art. 

\medskip\noindent
\textbf{Continuous-time Reinforcement Learning.} The seminal work of Doya \cite{doya2000reinforcement} introduced continuous-time RL. Since then, various approaches have been proposed to solve the Hamilton-Jacobi-Bellman (HJB) differential equation with the machine learning toolset. These methods can be divided into the trajectory and state-space-based methods. Trajectory-based methods solve the HJB along a trajectory to obtain the optimal trajectory. For example, path integral control uses the non-linear, control-affine dynamics with quadratic action costs to simplify the HJB to a linear partial differential equation~\cite{kappen2005linear, todorov2007linearly, theodorou2010reinforcement, pan2014model}. This differential equation can be transformed to a path-integral using the Feynman-Kac formulae. The path-integral can then be solved using Monte Carlo sampling to obtain the optimal state and action sequence. Recently, this approach has been used in combination with deep networks~\cite{rajagopal2016neural, pereira2019learning, pereira2020safe}. State-space-based methods solve the HJB globally to obtain an optimal non-linear controller applicable on the complete state domain. Classical approaches discretize the continuous spaces into a grid and solve the HJB or the robust Hamilton-Jacobi-Isaac (HJI) using a PDE solver \cite{bansal2017hamilton}. In contrast, machine learning-based methods use function approximation and sampled states to solve the HJB. For example, regression-based approaches solved the HJB by fitting a radial-basis-function networks \cite{doya2000reinforcement}, deep networks \cite{tassa2007least, lutter2019hjb, kim2020L4DC, kim2020hamilton}, kernels \cite{hennig2011optimal} or polynomial functions \cite{yang2014reinforcement, liu2014neural} to minimize the HJB residual.

\medskip\noindent
The presented work is closely related to the work of Kim et. al. \cite{kim2020L4DC, kim2020hamilton}. While we present a model-based approach that leverages the dynamics model to perform value iteration, \cite{kim2020L4DC, kim2020hamilton} propose a model-free approach that uses Q-iteration. To obtain the Q-function, which usually does not exist for continuous-time RL~\cite{doya2000reinforcement}, the authors incorporate a Lipschitz constraint that limits the change of the actions. Therefore, one does not control the actions but the change in the action. In this case, the optimal change in the action is the rescaled gradient of the Q-function w.r.t. the actions.

\medskip \noindent
\textbf{Fitted Value Iteration.}
Fitted Value Iteration (FVI)~\cite{boyan1994generalization, baird1995residual, tsitsiklis1996feature, munos2008finite} and the model-free counterpart fitted Q-Iteration~\cite{tesauro1992practical, ernst2005tree, massoud2009regularized, riedmiller2005neural, mnih2015human} were previously only applicable to discrete actions and continuous states. These approaches were limited to discrete actions as they could not solve the maximum for continuous actions. Only QT-Opt~\cite{kalashnikov2018qt} and NAF \cite{gu2016continuous} applied fitted Q-Iteration to continuous actions. QT-Opt solves the maximization in each step using the particle-based cross-entropy method (CEM). However, this approach requires solving an expensive optimization problem within each step. In contrast, NAF uses a specific Q-function parametrization such that the Q-function is quadratic w.r.t. the actions. Therefore, the maximum can be easily computed due to the quadratic form. By leveraging the continuous-time formulation and the control-affine dynamics of many robotics systems, we showed that this maximization can be solved in closed form. Therefore, we extended FVI to continuous actions for high-frequency control tasks. For very low control frequencies this approximation might not be sufficient.

\medskip\noindent
For continuous actions, current RL methods use policy iteration (PI) rather than VI \cite{schulman2015high, lillicrap2015continuous}. PI evaluates the value function of the current policy and hence, uses the action of the policy to compute the value function target. Therefore, PI circumvents the maximization required of VI (\Eqref{eq:fvi_update}). In contrast to the PI methods, our proposed method is 'policy-free' as the value function directly implies the policy. Therefore, \ac{cfvi} and rFVI do not require the additional optimization to improve the policy as the PI-based methods.

\medskip \noindent
\textbf{Robust Policies for Sim2Real.} Learning robust policies to bridge the simulation to reality gap has been approached by (1) changing the optimization objective \cite{borkar2001sensitivity, chow2015risk,bharadhwaj2021csc}, (2) using an adversary to optimize the worst-case performance~\cite{tamar2013scaling, isaacs1999differential, bansal2017hamilton,harrison2017adapt, hsu2021safety, morimoto2005robust, pinto2017robust, pinto2017supervision, mandlekar2017adversarially, tessler2019action} and (3) randomizing the simulation~\cite{andrychowicz2020learning, muratore2021data, muratore2021data, xie2021dynamics, chebotar2019closing, ramos2019bayessim}. In this paper, we focus on the adversarial formulation which has been used for continuous control tasks. For example, Pinto et. al. \cite{pinto2017robust, pinto2017supervision} used a separate agent as an adversary controlling an additive control input. This adversary maximized the negative reward using a standard actor-critic learning algorithm. The agent and adversary do not share any information. Therefore, an additional optimization is required to optimize the adversary. Mandlekar et. al.~\cite{mandlekar2017adversarially} used the auxiliary loss to maximize the policy actions. In contrast to these approaches, our approach is model-based instead of model-free. The model allows us to express the adversarial perturbations using analytic expressions derived directly from the Hamilton-Jacobi-Isaacs (HJI) equation. Therefore, our approach shares knowledge between the actor and adversary due to a shared value function and requires no additional optimization. 

\medskip\noindent 
Our proposed approach is similar to Morimoto and Doya \cite{morimoto2005robust}. In contrast to this work, we extend the analytic solutions to state, action, observation, and model disturbances, do not require a control-affine disturbance model, and use the constrained formulation rather than the penalized formulation.

\subsection{Summary}
We have proposed continuous fitted value iteration (cFVI) and robust fitted value iteration (rFVI). These algorithms can be used to solve the Hamilton-Jacobi-Bellman (HJB) differential equation and the Hamilton-Jacobi-Isaacs (HJI) equation for continuous states and action spaces without grid-based samples. To derive these algorithms, we extended the existing derivations~\cite{lyshevski1998optimal, doya2000reinforcement, morimoto2005robust} of the optimal policy to a wider class of reward functions and introduced the solution for the optimal adversary. Instead of solving these equations directly using the regression techniques of machine learning as prior methods~\cite{doya2000reinforcement, tassa2007least, lutter2019hjb}, we used value iteration to obtain a more reliable optimization. Thereby, we also extended fitted value iteration to continuous actions and adversarial \ac{rl}. Previously, fitted value iteration was mainly applicable to discrete actions. 
The continuous control experiments showed that both algorithms can obtain the optimal policy and obtain identical reward as the deep reinforcement learning methods. Furthermore, the policies can be transferred to the physical systems. The \ac{rfvi} policies are more robust when transferred to the physical system by applying higher actions. In addition, we provided an extensive discussion of the shortcomings and proposed approaches for future work to address these limitations.

% use Section* for acknowledgment
\ifCLASSOPTIONcompsoc
  \section*{Acknowledgments}
\else
  \section*{Acknowledgment}
\fi

\noindent The research was partially conducted during the internship of M. Lutter at NVIDIA.
M. Lutter, B. Belousov and J. Peters received funding from the European Union’s Horizon 2020 research and innovation
program under grant agreement No \#640554 (SKILLS4ROBOTS).
A. Garg was partially supported by CIFAR AI Chair.
Furthermore, we want to thank the open-source projects SimuRLacra~\cite{simurlacra}, MushroomRL~\cite{deramo2020mushroomrl}, NumPy~\cite{numpy} and PyTorch~\cite{pytorch}.

\ifCLASSOPTIONcaptionsoff
  \newpage
\fi


% references Section

\bibliographystyle{IEEEtran}
\bibliography{IEEEabrv, refs}

\begin{thebibliography}{10}
\providecommand{\url}[1]{#1}
\csname url@samestyle\endcsname
\providecommand{\newblock}{\relax}
\providecommand{\bibinfo}[2]{#2}
\providecommand{\BIBentrySTDinterwordspacing}{\spaceskip=0pt\relax}
\providecommand{\BIBentryALTinterwordstretchfactor}{4}
\providecommand{\BIBentryALTinterwordspacing}{\spaceskip=\fontdimen2\font plus
\BIBentryALTinterwordstretchfactor\fontdimen3\font minus
  \fontdimen4\font\relax}
\providecommand{\BIBforeignlanguage}[2]{{%
\expandafter\ifx\csname l@#1\endcsname\relax
\typeout{** WARNING: IEEEtran.bst: No hyphenation pattern has been}%
\typeout{** loaded for the language `#1'. Using the pattern for}%
\typeout{** the default language instead.}%
\else
\language=\csname l@#1\endcsname
\fi
#2}}
\providecommand{\BIBdecl}{\relax}
\BIBdecl

\bibitem{liberzon2011calculus}
D.~Liberzon, \emph{Calculus of variations and optimal control theory: a concise
  introduction}.\hskip 1em plus 0.5em minus 0.4em\relax Princeton University
  Press, 2011.

\bibitem{carter2001foundations}
M.~Carter, \emph{Foundations of mathematical economics}.\hskip 1em plus 0.5em
  minus 0.4em\relax MIT press, 2001.

\bibitem{caputo2005foundations}
M.~R. Caputo and M.~R. Caputo, \emph{Foundations of dynamic economic analysis:
  optimal control theory and applications}.\hskip 1em plus 0.5em minus
  0.4em\relax Cambridge University Press, 2005.

\bibitem{lavalle2006planning}
S.~M. LaValle, \emph{Planning algorithms}.\hskip 1em plus 0.5em minus
  0.4em\relax Cambridge university press, 2006.

\bibitem{chen2015safe}
M.~Chen, J.~F. Fisac, S.~Sastry, and C.~J. Tomlin, ``Safe sequential path
  planning of multi-vehicle systems via double-obstacle
  {Hamilton-Jacobi-Isaacs} variational inequality,'' in \emph{European Control
  Conference (ECC)}, 2015.

\bibitem{bansal2017hamilton}
S.~Bansal, M.~Chen, S.~Herbert, and C.~J. Tomlin, ``{Hamilton-Jacobi}
  reachability: A brief overview and recent advances,'' \emph{Conference on
  Decision and Control (CDC)}, 2017.

\bibitem{fisac2018general}
J.~F. Fisac, A.~K. Akametalu, M.~N. Zeilinger, S.~Kaynama, J.~Gillula, and
  C.~J. Tomlin, ``A general safety framework for learning-based control in
  uncertain robotic systems,'' \emph{IEEE Transactions on Automatic Control},
  2018.

\bibitem{doya2000reinforcement}
K.~Doya, ``Reinforcement learning in continuous time and space,'' \emph{Neural
  computation}, 2000.

\bibitem{morimoto2005robust}
J.~Morimoto and K.~Doya, ``Robust reinforcement learning,'' \emph{Neural
  computation}, 2005.

\bibitem{tassa2007least}
Y.~Tassa and T.~Erez, ``Least squares solutions of the {HJB} equation with
  neural network value-function approximators,'' \emph{IEEE Transactions on
  Neural Networks}, 2007.

\bibitem{lutter2019hjb}
M.~Lutter, B.~Belousov, K.~Listmann, D.~Clever, and J.~Peters, ``{HJB} optimal
  feedback control with deep differential value functions and action
  constraints,'' in \emph{Conference on Robot Learning (CoRL)}, 2019.

\bibitem{yang2014reinforcement}
X.~Yang, D.~Liu, and D.~Wang, ``Reinforcement learning for adaptive optimal
  control of unknown continuous-time nonlinear systems with input
  constraints,'' \emph{International Journal of Control}, 2014.

\bibitem{liu2014neural}
D.~Liu, D.~Wang, F.-Y. Wang, H.~Li, and X.~Yang, ``Neural-network-based online
  hjb solution for optimal robust guaranteed cost control of continuous-time
  uncertain nonlinear systems,'' \emph{IEEE transactions on cybernetics}, 2014.

\bibitem{kim2020L4DC}
J.~Kim and I.~Yang, ``{Hamilton-Jacobi-Bellman} equations for {Q-Learning} in
  continuous time,'' in \emph{Learning for Dynamics and Control}, 2020.

\bibitem{kim2020hamilton}
J.~Kim, J.~Shin, and I.~Yang, ``{Hamilton-Jacobi} deep {Q-Learning} for
  deterministic continuous-time systems with lipschitz continuous controls,''
  \emph{arXiv preprint arXiv:2010.14087}, 2020.

\bibitem{lutter2021cfvi}
M.~Lutter, S.~Mannor, J.~Peters, D.~Fox, and A.~Garg, ``{Value Iteration in
  Continuous Actions, States and Time},'' in \emph{International Conference on
  Machine Learning (ICML)}, 2021.

\bibitem{lutter2021robust}
------, ``Robust value iteration for continuous control tasks,''
  \emph{Robotics: Science and Systems (RSS)}, 2021.

\bibitem{lyshevski1998optimal}
S.~E. Lyshevski, ``Optimal control of nonlinear continuous-time systems: design
  of bounded controllers via generalized nonquadratic functionals,'' in
  \emph{American Control Conference (ACC)}.\hskip 1em plus 0.5em minus
  0.4em\relax IEEE, 1998.

\bibitem{kirk2004optimal}
D.~E. Kirk, \emph{Optimal control theory: an introduction}.\hskip 1em plus
  0.5em minus 0.4em\relax Courier Corporation, 1970.

\bibitem{kappen2005linear}
H.~J. Kappen, ``Linear theory for control of nonlinear stochastic systems,''
  \emph{Physical review letters}, 2005.

\bibitem{todorov2007linearly}
E.~Todorov, ``Linearly-solvable markov decision problems,'' in \emph{Advances
  in neural information processing systems}, 2007.

\bibitem{isaacs1999differential}
R.~Isaacs, \emph{Differential games: a mathematical theory with applications to
  warfare and pursuit, control and optimization}.\hskip 1em plus 0.5em minus
  0.4em\relax Courier Corporation, 1999.

\bibitem{zhang2020robust}
H.~Zhang, H.~Chen, C.~Xiao, B.~Li, D.~Boning, and C.-J. Hsieh, ``Robust deep
  reinforcement learning against adversarial perturbations on observations,''
  \emph{arXiv preprint arXiv:2003.08938}, 2020.

\bibitem{heger1994consideration}
M.~Heger, ``Consideration of risk in reinforcement learning,'' in \emph{Machine
  Learning Proceedings}.\hskip 1em plus 0.5em minus 0.4em\relax Elsevier, 1994.

\bibitem{mandlekar2017adversarially}
A.~Mandlekar, Y.~Zhu, A.~Garg, L.~Fei-Fei, and S.~Savarese, ``Adversarially
  robust policy learning: Active construction of physically-plausible
  perturbations,'' in \emph{International Conference on Intelligent Robots and
  Systems (IROS)}, 2017.

\bibitem{littman1994markov}
M.~L. Littman, ``Markov games as a framework for multi-agent reinforcement
  learning,'' in \emph{Machine learning proceedings}.\hskip 1em plus 0.5em
  minus 0.4em\relax Elsevier, 1994.

\bibitem{nilim2005robust}
A.~Nilim and L.~El~Ghaoui, ``Robust control of markov decision processes with
  uncertain transition matrices,'' \emph{Operations Research}, 2005.

\bibitem{pinto2017robust}
L.~Pinto, J.~Davidson, R.~Sukthankar, and A.~Gupta, ``Robust adversarial
  reinforcement learning,'' in \emph{International Conference on Machine
  Learning (ICML)}, 2017.

\bibitem{pinto2017supervision}
L.~Pinto, J.~Davidson, and A.~Gupta, ``Supervision via competition: Robot
  adversaries for learning tasks,'' in \emph{International Conference on
  Robotics and Automation (ICRA)}, 2017.

\bibitem{tessler2019action}
C.~Tessler, Y.~Efroni, and S.~Mannor, ``Action robust reinforcement learning
  and applications in continuous control,'' in \emph{International Conference
  on Machine Learning (ICML)}, 2019.

\bibitem{pattanaik2017robust}
A.~Pattanaik, Z.~Tang, S.~Liu, G.~Bommannan, and G.~Chowdhary, ``Robust deep
  reinforcement learning with adversarial attacks,'' \emph{arXiv preprint
  arXiv:1712.03632}, 2017.

\bibitem{gleave2019adversarial}
A.~Gleave, M.~Dennis, C.~Wild, N.~Kant, S.~Levine, and S.~Russell,
  ``Adversarial policies: Attacking deep reinforcement learning,'' \emph{arXiv
  preprint arXiv:1905.10615}, 2019.

\bibitem{xu2012robustness}
H.~Xu and S.~Mannor, ``Robustness and generalization,'' \emph{Machine
  learning}, 2012.

\bibitem{astrom2008event}
K.~J. Astr{\"o}m, ``Event based control,'' in \emph{Analysis and design of
  nonlinear control systems}.\hskip 1em plus 0.5em minus 0.4em\relax Springer,
  2008.

\bibitem{de2000elucidation}
J.~A. De~Don{\'a} and G.~C. Goodwin, ``Elucidation of the state-space regions
  wherein model predictive control and anti-windup strategies achieve identical
  control policies,'' in \emph{American Control Conference (ACC)}.\hskip 1em
  plus 0.5em minus 0.4em\relax IEEE, 2000.

\bibitem{boyd2004convex}
S.~Boyd and L.~Vandenberghe, \emph{Convex optimization}.\hskip 1em plus 0.5em
  minus 0.4em\relax Cambridge University Press, 2004.

\bibitem{ching2010quasilinear}
S.~Ching, Y.~Eun, C.~Gokcek, P.~T. Kabamba, and S.~M. Meerkov,
  \emph{Quasilinear Control: Performance Analysis and Design of Feedback
  Systems with Nonlinear Sensors and Actuators}.\hskip 1em plus 0.5em minus
  0.4em\relax Cambridge University Press, 2010.

\bibitem{fleming2006controlled}
W.~H. Fleming and H.~M. Soner, \emph{Controlled Markov processes and viscosity
  solutions}.\hskip 1em plus 0.5em minus 0.4em\relax Springer Science \&
  Business Media, 2006.

\bibitem{boyan1994generalization}
J.~Boyan and A.~Moore, ``Generalization in reinforcement learning: Safely
  approximating the value function,'' \emph{Advances in neural information
  processing systems (NeurIPS)}, 1994.

\bibitem{baird1995residual}
L.~Baird, ``Residual algorithms: Reinforcement learning with function
  approximation,'' in \emph{Machine Learning Proceedings}.\hskip 1em plus 0.5em
  minus 0.4em\relax Elsevier, 1995.

\bibitem{tsitsiklis1996feature}
J.~N. Tsitsiklis and B.~Van~Roy, ``Feature-based methods for large scale
  dynamic programming,'' \emph{Machine Learning}, 1996.

\bibitem{munos2008finite}
R.~Munos and C.~Szepesv{\'a}ri, ``Finite-time bounds for fitted value
  iteration,'' \emph{Journal of Machine Learning Research}, 2008.

\bibitem{bellman1957dynamic}
R.~Bellman, \emph{Dynamic Programming}.\hskip 1em plus 0.5em minus 0.4em\relax
  Princeton University Press, 1957.

\bibitem{puterman1994markov}
M.~L. Puterman, \emph{Markov decision processes: discrete stochastic dynamic
  programming}.\hskip 1em plus 0.5em minus 0.4em\relax John Wiley \& Sons,
  1994.

\bibitem{tesauro1992practical}
G.~Tesauro, ``Practical issues in temporal difference learning,'' \emph{Machine
  learning}, 1992.

\bibitem{ernst2005tree}
D.~Ernst, P.~Geurts, and L.~Wehenkel, ``Tree-based batch mode reinforcement
  learning,'' \emph{Journal of Machine Learning Research}, 2005.

\bibitem{massoud2009regularized}
A.~massoud Farahmand, M.~Ghavamzadeh, C.~Szepesv{\'a}ri, and S.~Mannor,
  ``Regularized fitted q-iteration for planning in continuous-space markovian
  decision problems,'' in \emph{American Control Conference (ACC)}.\hskip 1em
  plus 0.5em minus 0.4em\relax IEEE, 2009.

\bibitem{riedmiller2005neural}
M.~Riedmiller, ``Neural fitted q iteration--first experiences with a data
  efficient neural reinforcement learning method,'' in \emph{European
  Conference on Machine Learning}.\hskip 1em plus 0.5em minus 0.4em\relax
  Springer, 2005.

\bibitem{mnih2015human}
V.~Mnih, K.~Kavukcuoglu, D.~Silver, A.~A. Rusu, J.~Veness, M.~G. Bellemare,
  A.~Graves, M.~Riedmiller, A.~K. Fidjeland, G.~Ostrovski \emph{et~al.},
  ``Human-level control through deep reinforcement learning,'' \emph{Nature},
  2015.

\bibitem{sutton1998introduction}
R.~S. Sutton, A.~G. Barto \emph{et~al.}, \emph{Introduction to reinforcement
  learning}.\hskip 1em plus 0.5em minus 0.4em\relax MIT press Cambridge, 1998.

\bibitem{schulman2015high}
J.~Schulman, P.~Moritz, S.~Levine, M.~Jordan, and P.~Abbeel, ``High-dimensional
  continuous control using generalized advantage estimation,'' \emph{arXiv
  preprint arXiv:1506.02438}, 2015.

\bibitem{schulman2017proximal}
J.~Schulman, F.~Wolski, P.~Dhariwal, A.~Radford, and O.~Klimov, ``Proximal
  policy optimization algorithms,'' \emph{arXiv preprint arXiv:1707.06347},
  2017.

\bibitem{feinberg2018model}
V.~Feinberg, A.~Wan, I.~Stoica, M.~I. Jordan, J.~E. Gonzalez, and S.~Levine,
  ``Model-based value estimation for efficient model-free reinforcement
  learning,'' \emph{arXiv preprint arXiv:1803.00101}, 2018.

\bibitem{buckman2018sample}
J.~Buckman, D.~Hafner, G.~Tucker, E.~Brevdo, and H.~Lee, ``Sample-efficient
  reinforcement learning with stochastic ensemble value expansion,'' in
  \emph{Advances in Neural Information Processing Systems (NeurIPS)}, 2018.

\bibitem{barto1995learning}
A.~G. Barto, S.~J. Bradtke, and S.~P. Singh, ``Learning to act using real-time
  dynamic programming,'' \emph{Artificial intelligence}, 1995.

\bibitem{gulcehre2020rl}
C.~Gulcehre, Z.~Wang, A.~Novikov, T.~L. Paine, S.~G. Colmenarejo, K.~Zolna,
  R.~Agarwal, J.~Merel, D.~Mankowitz, C.~Paduraru \emph{et~al.}, ``Rl
  unplugged: Benchmarks for offline reinforcement learning,'' \emph{arXiv
  preprint arXiv:2006.13888}, 2020.

\bibitem{muratore2018domain}
F.~Muratore, F.~Treede, M.~Gienger, and J.~Peters, ``Domain randomization for
  simulation-based policy optimization with transferability assessment,'' in
  \emph{Conference on Robot Learning (CoRL)}, 2018.

\bibitem{muratore2021data}
F.~Muratore, C.~Eilers, M.~Gienger, and J.~Peters, ``Data-efficient domain
  randomization with bayesian optimization,'' \emph{IEEE Robotics and
  Automation Letters (RAL)}, 2021.

\bibitem{lutter2021value}
M.~Lutter, S.~Mannor, J.~Peters, D.~Fox, and A.~Garg, ``Value iteration in
  continuous actions, states and time,'' \emph{International Conference on
  Machine Learning (ICML)}, 2021.

\bibitem{khalil2002nonlinear}
H.~K. Khalil and J.~W. Grizzle, \emph{Nonlinear systems}.\hskip 1em plus 0.5em
  minus 0.4em\relax Prentice hall Upper Saddle River, NJ, 2002.

\bibitem{gu2016continuous}
S.~Gu, T.~Lillicrap, I.~Sutskever, and S.~Levine, ``Continuous deep q-learning
  with model-based acceleration,'' in \emph{International Conference on Machine
  Learning (ICML)}, 2016.

\bibitem{berkenkamp2017safe}
F.~Berkenkamp, M.~Turchetta, A.~Schoellig, and A.~Krause, ``Safe model-based
  reinforcement learning with stability guarantees,'' in \emph{Advances in
  neural information processing systems}, 2017.

\bibitem{richards2018lyapunov}
S.~M. Richards, F.~Berkenkamp, and A.~Krause, ``The lyapunov neural network:
  Adaptive stability certification for safe learning of dynamical systems,''
  \emph{arXiv preprint arXiv:1808.00924}, 2018.

\bibitem{kolter2019learning}
J.~Z. Kolter and G.~Manek, ``Learning stable deep dynamics models,'' in
  \emph{Advances in Neural Information Processing Systems (NeurIPS)}, 2019.

\bibitem{chang2019neural}
Y.-C. Chang, N.~Roohi, and S.~Gao, ``Neural lyapunov control,'' in
  \emph{Advances in Neural Information Processing Systems (NeurIPS)}, 2019.

\bibitem{bharadhwaj2021csc}
H.~Bharadhwaj, A.~Kumar, N.~Rhinehart, S.~Levine, F.~Shkurti, and A.~Garg,
  ``Conservative safety critics for exploration,'' in \emph{International
  Conference on Learning Representations (ICLR)}, 2021.

\bibitem{quanser}
Quanser, ``{Quanser} courseware and resources,''
  \url{https://www.quanser.com/solution/control-systems/}, 2018.

\bibitem{lillicrap2015continuous}
T.~P. Lillicrap, J.~J. Hunt, A.~Pritzel, N.~Heess, T.~Erez, Y.~Tassa,
  D.~Silver, and D.~Wierstra, ``Continuous control with deep reinforcement
  learning,'' \emph{arXiv preprint arXiv:1509.02971}, 2015.

\bibitem{haarnoja2018soft}
T.~Haarnoja, A.~Zhou, P.~Abbeel, and S.~Levine, ``Soft actor-critic: Off-policy
  maximum entropy deep reinforcement learning with a stochastic actor,'' in
  \emph{International Conference on Machine Learning (ICML)}, 2018.

\bibitem{lutter2018deep}
M.~Lutter, C.~Ritter, and J.~Peters, ``Deep lagrangian networks: Using physics
  as model prior for deep learning,'' in \emph{International Conference on
  Learning Representations}, 2019.

\bibitem{Lutter2019Energy}
M.~Lutter and J.~Peters, ``Deep lagrangian networks for end-to-end learning of
  energy-based control for under-actuated systems,'' in \emph{International
  Conference on Intelligent Robots and Systems (IROS)}, 2019.

\bibitem{zhong2019symplectic}
Y.~D. Zhong, B.~Dey, and A.~Chakraborty, ``Symplectic ode-net: Learning
  hamiltonian dynamics with control,'' \emph{arXiv preprint arXiv:1909.12077},
  2019.

\bibitem{theodorou2010reinforcement}
E.~Theodorou, J.~Buchli, and S.~Schaal, ``Reinforcement learning of motor
  skills in high dimensions: A path integral approach,'' in \emph{International
  Conference on Robotics and Automation (ICRA)}.\hskip 1em plus 0.5em minus
  0.4em\relax IEEE, 2010.

\bibitem{pan2014model}
Y.~Pan, E.~A. Theodorou, and M.~Kontitsis, ``Model-based path integral
  stochastic control: A bayesian nonparametric approach,'' \emph{arXiv preprint
  arXiv:1412.3038}, 2014.

\bibitem{rajagopal2016neural}
K.~Rajagopal, S.~N. Balakrishnan, and J.~R. Busemeyer, ``Neural network-based
  solutions for stochastic optimal control using path integrals,'' \emph{IEEE
  transactions on neural networks and learning systems}, 2016.

\bibitem{pereira2019learning}
M.~Pereira, Z.~Wang, I.~Exarchos, and E.~Theodorou, ``Learning deep stochastic
  optimal control policies using forward-backward sdes,'' in \emph{Robotics:
  science and systems}, 2019.

\bibitem{pereira2020safe}
M.~A. Pereira, Z.~Wang, I.~Exarchos, and E.~A. Theodorou, ``Safe optimal
  control using stochastic barrier functions and deep forward-backward sdes,''
  \emph{arXiv preprint arXiv:2009.01196}, 2020.

\bibitem{hennig2011optimal}
P.~Hennig, ``Optimal reinforcement learning for gaussian systems,'' in
  \emph{Advances in Neural Information Processing Systems}, 2011.

\bibitem{kalashnikov2018qt}
D.~Kalashnikov, A.~Irpan, P.~Pastor, J.~Ibarz, A.~Herzog, E.~Jang, D.~Quillen,
  E.~Holly, M.~Kalakrishnan, V.~Vanhoucke \emph{et~al.}, ``Qt-opt: Scalable
  deep reinforcement learning for vision-based robotic manipulation,''
  \emph{arXiv preprint arXiv:1806.10293}, 2018.

\bibitem{borkar2001sensitivity}
V.~S. Borkar, ``A sensitivity formula for risk-sensitive cost and the
  actor--critic algorithm,'' \emph{Systems \& Control Letters}, 2001.

\bibitem{chow2015risk}
Y.~Chow, A.~Tamar, S.~Mannor, and M.~Pavone, ``Risk-sensitive and robust
  decision-making: a cvar optimization approach,'' \emph{arXiv preprint
  arXiv:1506.02188}, 2015.

\bibitem{tamar2013scaling}
A.~Tamar, H.~Xu, and S.~Mannor, ``Scaling up robust mdps by reinforcement
  learning,'' \emph{arXiv preprint arXiv:1306.6189}, 2013.

\bibitem{harrison2017adapt}
J.~Harrison*, A.~Garg*, B.~Ivanovic, Y.~Zhu, S.~Savarese, L.~Fei-Fei, and
  M.~Pavone (*~equal contribution), ``{AdaPT: Zero-Shot Adaptive Policy
  Transfer for Stochastic Dynamical Systems},'' in \emph{International
  Symposium on Robotics Research (ISRR)}.\hskip 1em plus 0.5em minus
  0.4em\relax Springer STAR, 2017.

\bibitem{hsu2021safety}
K.-C. Hsu, V.~Rubies-Royo, C.~J. Tomlin, and J.~F. Fisac, ``Safety and liveness
  guarantees through reach-avoid reinforcement learning,'' \emph{Robotics:
  Science and Systems}, 2021.

\bibitem{andrychowicz2020learning}
O.~M. Andrychowicz, B.~Baker, M.~Chociej, R.~Jozefowicz, B.~McGrew,
  J.~Pachocki, A.~Petron, M.~Plappert, G.~Powell, A.~Ray \emph{et~al.},
  ``Learning dexterous in-hand manipulation,'' \emph{The International Journal
  of Robotics Research}, 2020.

\bibitem{xie2021dynamics}
Z.~Xie, X.~Da, M.~{van de Panne}, B.~Babich, and A.~Garg, ``{Dynamics
  Randomization Revisited: A Case Study for Quadrupedal Locomotion},'' in
  \emph{IEEE International Conference on Robotics and Automation (ICRA)}, 2021.

\bibitem{chebotar2019closing}
Y.~Chebotar, A.~Handa, V.~Makoviychuk, M.~Macklin, J.~Issac, N.~Ratliff, and
  D.~Fox, ``Closing the sim-to-real loop: Adapting simulation randomization
  with real world experience,'' 2019.

\bibitem{ramos2019bayessim}
F.~Ramos, R.~C. Possas, and D.~Fox, ``Bayessim: adaptive domain randomization
  via probabilistic inference for robotics simulators,'' 2019.

\bibitem{simurlacra}
F.~Muratore, ``Simurlacra - a framework for reinforcement learning from
  randomized simulations,'' \url{https://github.com/famura/SimuRLacra}, 2020.

\bibitem{deramo2020mushroomrl}
C.~D'Eramo, D.~Tateo, A.~Bonarini, M.~Restelli, and J.~Peters, ``{MushroomRL}:
  Simplifying reinforcement learning research,''
  \url{https://github.com/MushroomRL/mushroom-rl}, 2020.

\bibitem{numpy}
C.~R. Harris, K.~J. Millman, S.~J. van~der Walt, R.~Gommers, P.~Virtanen,
  D.~Cournapeau, E.~Wieser, J.~Taylor, S.~Berg, N.~J. Smith, R.~Kern, M.~Picus,
  S.~Hoyer, M.~H. van Kerkwijk, M.~Brett, A.~Haldane, J.~F. del R{\'{i}}o,
  M.~Wiebe, P.~Peterson, P.~G{\'{e}}rard-Marchant, K.~Sheppard, T.~Reddy,
  W.~Weckesser, H.~Abbasi, C.~Gohlke, and T.~E. Oliphant, ``Array programming
  with {NumPy},'' \emph{Nature}, 2020.

\bibitem{pytorch}
A.~Paszke, S.~Gross, F.~Massa, A.~Lerer, J.~Bradbury, G.~Chanan, T.~Killeen,
  Z.~Lin, N.~Gimelshein, L.~Antiga, A.~Desmaison, A.~Kopf, E.~Yang, Z.~DeVito,
  M.~Raison, A.~Tejani, S.~Chilamkurthy, B.~Steiner, L.~Fang, J.~Bai, and
  S.~Chintala, ``{PyTorch}: An imperative style, high-performance deep learning
  library,'' in \emph{Advances in Neural Information Processing Systems 32},
  2019.

\end{thebibliography}


% \end{thebibliography}

% biography Section

\begin{IEEEbiography}[{\includegraphics[width=1in,height=1.25in,clip,keepaspectratio]{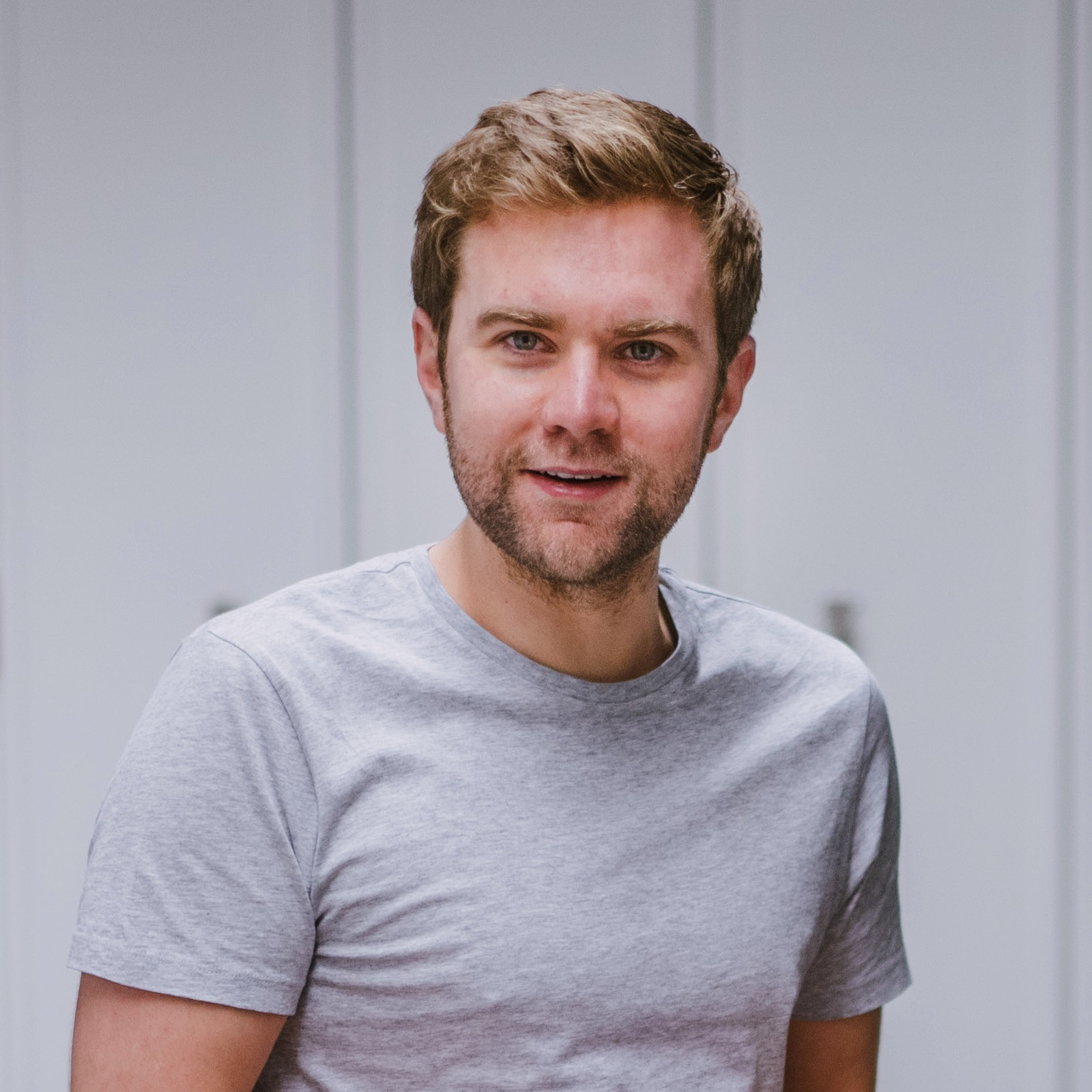}}]{Michael Lutter}
Michael Lutter is a Ph.D. student with Jan Peters at the Institute for Intelligent Autonomous Systems (IAS) TU Darmstadt since July 2017. During his Ph.D., Michael focuses on inductive biases for robot learning. During his Ph.D. completed a research internship at DeepMind, NVIDIA Research and received multiple awards for his research including the AI newcomer award (2019) of the German computer science foundation and was NVIDIA graduate fellowship finalist (2019). Prior to this, Michael received his M.Sc. in Electrical Engineering from the Technical University of Munich (TUM) amd worked for the Human Brain Project on bio-inspired learning for robotics.
\end{IEEEbiography}

\begin{IEEEbiography}[{\includegraphics[width=1in,height=1.25in,clip,keepaspectratio]{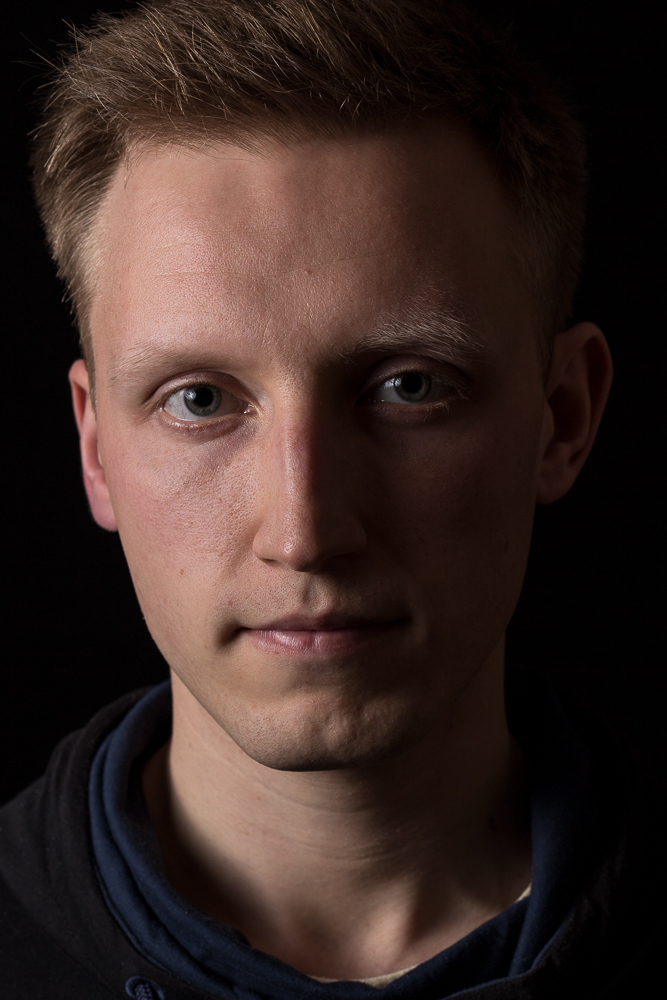}}]{Boris Belousov}
Boris Belousov received his M.Sc. degree in Electrical Engineering from the Friedrich–Alexander University Erlangen–Nürnberg in 2016. Currently, he is pursuing his Ph.D. at the Intelligent Autonomous Systems Group at the Computer Science Department of the Technical University of Darmstadt. His research interests center around Bayesian decision theory for optimal control under uncertainty in robotics.
\end{IEEEbiography}

\begin{IEEEbiography}[{\includegraphics[width=1in,height=1.25in,clip,keepaspectratio]{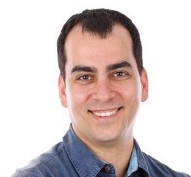}}]{Shie Mannor}
Shie Mannor (S'00-M'03-SM-09') received his Ph.D. degree in electrical engineering from the Technion-Israel Institute of Technology, Haifa, Israel, in 2002. From 2002 to 2004, he was a Fulbright scholar and a postdoctoral associate at M.I.T. He was with the Department of Electrical and Computer Engineering at McGill University from 2004 to 2010 where he was the Canada Research chair in Machine Learning. He has been with the Faculty of Electrical Engineering at the Technion since 2008 where he is currently a professor and co-director of the Machine Learning and Intelligent Systems center. He is also Distinguished Scientist at NVIDIA. Shie was the program chair of COLT in 2013. His research interests include machine learning and pattern recognition, planning and control focusing on reinforcement learning, multi-agent systems, and communications. 
\end{IEEEbiography}

\begin{IEEEbiography}[{\includegraphics[width=1in,height=1.25in,clip,keepaspectratio]{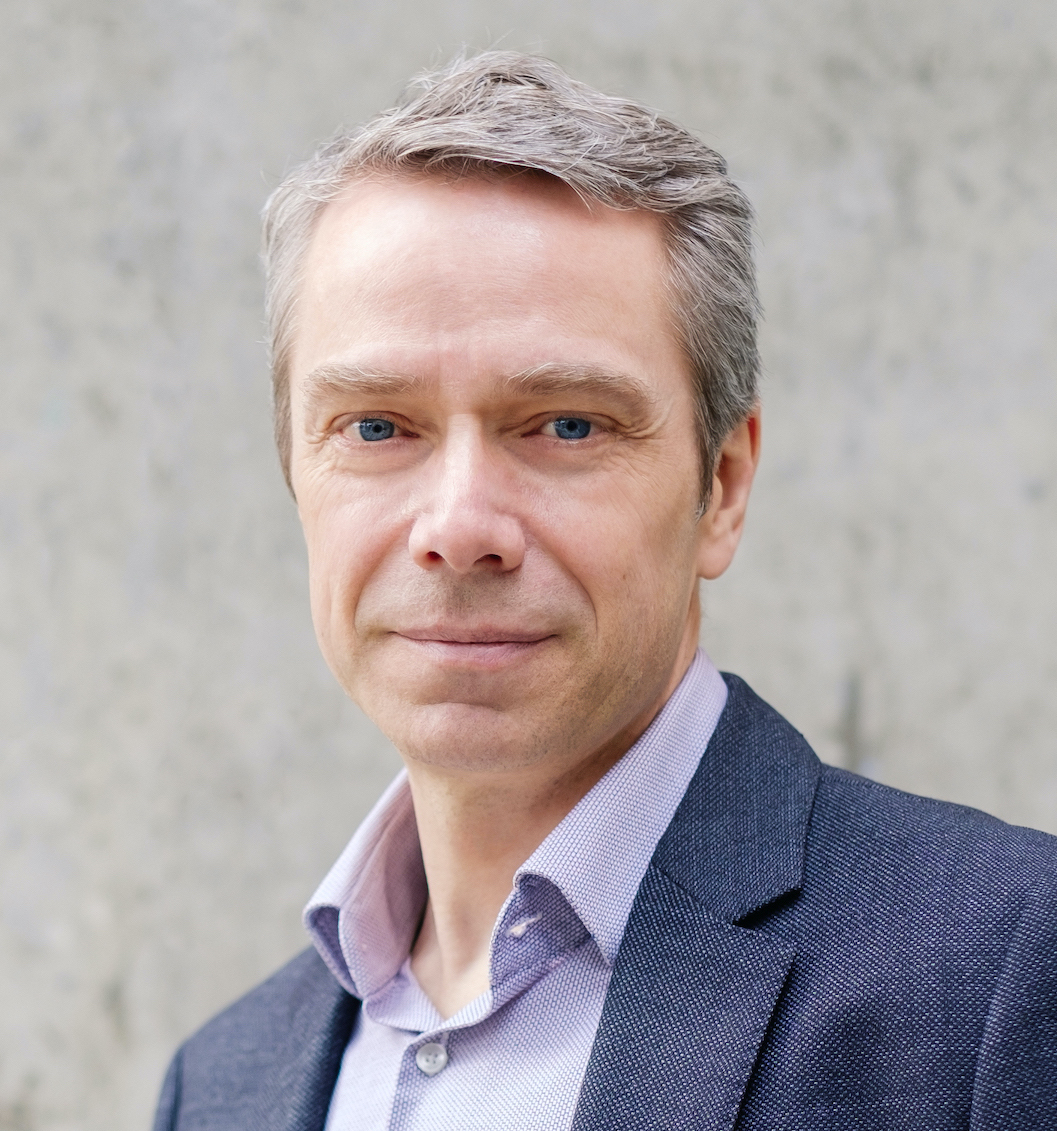}}]{Dieter Fox}
Dieter Fox received the PhD degree from the University of Bonn,
Germany. He is a professor in the Allen School of Computer Science \&
Engineering at the University of Washington, where he heads the UW
Robotics and State Estimation Lab.  He is also Senior Director of
Robotics Research at NVIDIA.  His research is in robotics and
artificial intelligence, with a focus on state estimation and
perception applied to problems such as mapping, object detection and
tracking, manipulation, and activity recognition. He has published
more than 200 technical papers and is co-author of the textbook
"Probabilistic Robotics". He is a Fellow of the IEEE , AAAI, and ACM, and
recipient of the IEEE RAS Pioneer Award.  He was an editor of the
IEEE Transactions on Robotics, program co-chair of the 2008 AAAI
Conference on Artificial Intelligence, and program chair of the 2013
Robotics: Science and Systems conference.
\end{IEEEbiography}

\begin{IEEEbiography}[{\includegraphics[width=1in,height=1.25in,clip,keepaspectratio]{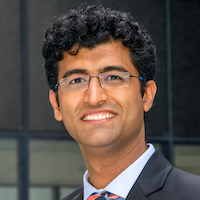}}]{Animesh Garg}
Animesh Garg is a CIFAR Chair Assistant Professor of Computer Science at University of Toronto, a Faculty Member at the Vector Institute, and a Senior Research Scientist at Nvidia. He earned a PhD from UC, Berkeley and was a postdoc at Stanford AI lab. He works on the Algorithmic Foundations for Generalizable Autonomy, to enable AI-based robots to work alongside humans.
\end{IEEEbiography}

\begin{IEEEbiography}[{\includegraphics[width=1in,height=1.25in,clip,keepaspectratio]{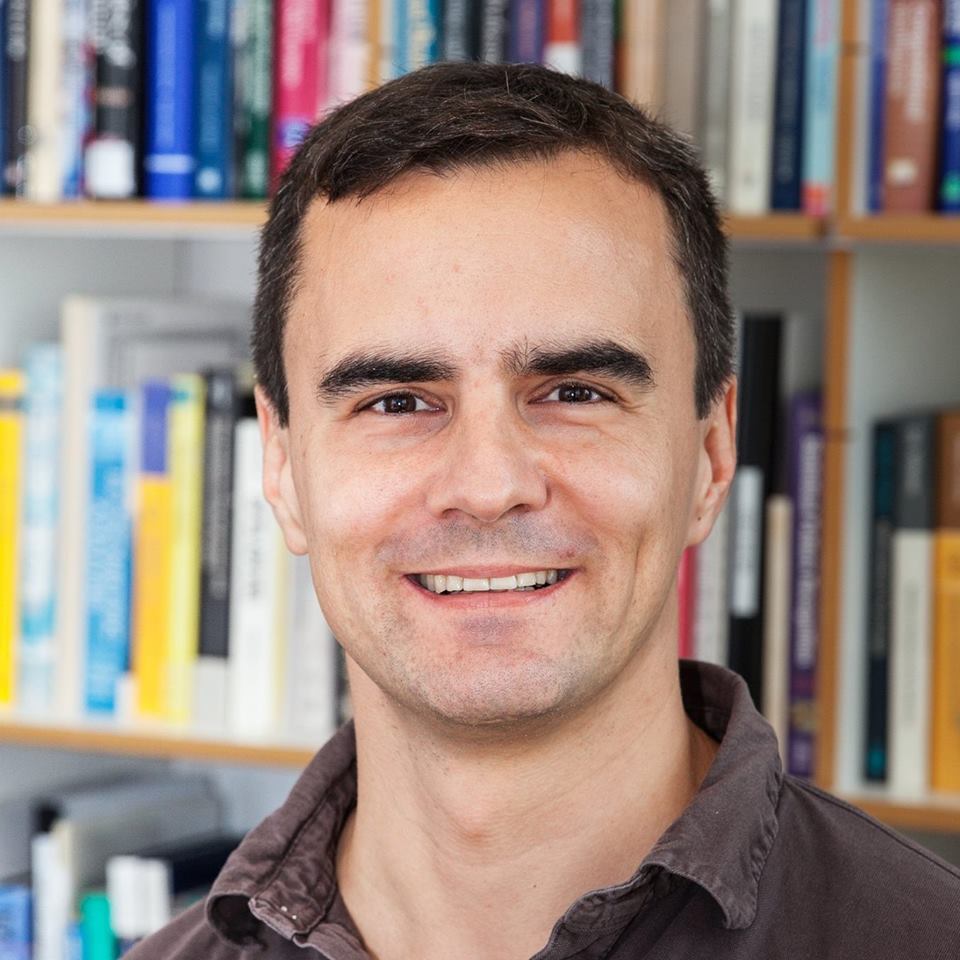}}]{Jan Peters}
s is a full professor (W3) in intelligent
autonomous systems, Computer Science Depart-
ment, Technical University of Darmstadt. He has received
the Dick Volz Best 2007 US PhD Thesis Runner-
Up Award, Robotics: Science \& Systems - Early
Career Spotlight, INNS Young Investigator Award,
and IEEE Robotics \& Automation Society’s Early
Career Award as well as numerous best paper awards. In 2015, he
received an ERC Starting Grant and in 2019, he was appointed as an
IEEE fellow.
\end{IEEEbiography}


\vfill

% \enlargethispage{-5in}

\newpage
\appendices
\section{Optimal Policy Proof - Theorem \ref{theorem:opt_policy}}

\begin{theorem*}
If the dynamics are control affine (\Eqref{eq:affine_dyn}), the reward is separable w.r.t. to state and action (\Eqref{eq:seperable_rwd}) and the action cost $g_c$ is positive definite and strictly convex, the continuous time optimal policy $\pi^{k}$ w.r.t. $V^{k}$ is described by
\begin{gather*}
    \pi^{*}(\vx)  = \nabla \tilde{g}_c \left( \mB(\vx)^{\top} \nabla_xV^{*} \right) \label{eq:optimal_policy}
\end{gather*}
with the convex conjugate $\tilde{g}$ of $g$ and the Jacobian of $V$ w.r.t. the system state $\nabla_x V$.
\end{theorem*} 

\begin{proof}
This proof follows the derivation our prior work \cite{lutter2019hjb} that generalized the special case described by \cite{doya2000reinforcement}. Starting from the \acf{hjb} equations and substituting the control affine dynamics yields 
\begin{align*}
\rho \: V^{*}(\vx) &= \max_{\vu} \:\: r(\vx, \vu) \:\: + f_c(\vx, \vu)^{\top} \: \nabla_xV^{*}, \\
&= \max_{\vu} \:\: q_c(\vx) - g_c( \vu) \:\: + \left[\va(\vx) + \mB(\vx) \vu \right] \: \nabla_x V^{*} \\
&= \max_{\vu} \:\: \left[\vu^{\top} \mB^{\top} \nabla_x V^{*} - g_c( \vu) \right] + q_c(\vx)  + \va^{\top} \: \nabla_x V^{*}.
\end{align*}
Therefore, the continuous time optimal action is defined as 
\begin{align*} 
    \vu^{*} = \argmax_{\vu} \:  \vu^{\top} \mB^{\top} \nabla_{x}V^{*} - g_{c}(\vu).
\end{align*}
This optimization can be solved analytically as $g_c$ is strictly convex and hence $\nabla g_{c}(\vu) = \vw$ is invertible, i.e., $\vu = \left[ \nabla g_{c}\right]^{-1}(\vw) \coloneqq \nabla \tilde{g}_{c}(\vw)$ with the convex conjugate $\tilde{g}$. The optimal action is described by
\begin{align*}
&\mB(\vx)^{\top} \nabla_{x}V^{*} - \nabla g_c(\vu) \coloneqq 0 &
&
\Rightarrow & \vu^{*} = \nabla \tilde{g}_c \left( \mB(\vx)^{\top} \nabla_{x}V^{*} \right).
\end{align*}
Therefore, the optimal action follows the gradient of the Value function upwards by projecting the gradient on using the control matrix and rescaling the magnitude using the action cost. 
\end{proof}

\section{Optimal Adversary Proof - Theorem 2}
\subsection{State Disturbance Proof}
\begin{theorem*}
For the adversarial state disturbance (\Eqref{eq:dyn_adv_state}) with bounded in signal energy (\Eqref{eq:signal_energy}), the optimal continuous time policy $\pi^{k}$ and state disturbance $\vxi^{k}_{x}$ is described by
\begin{align*}
\pi^{*}(\vx) &= \nabla \tilde{g} \left(\mB(\vx)^{\top} \nabla_x V^{*}\right) & \vxi_{x}^{*} &= - \alpha \frac{\nabla_x V^{*}}{\| \nabla_x V^{*} \|_2}.
\end{align*}
\end{theorem*}

\noindent
\emph{Proof.} \Eqref{eq:hji} can be formulated with the explicit constraint instead of the infimum. The constrained \acf{hji} equation with an bounded signal energy constraint is described by 
\begin{gather*}
\rho \: V^{*}(\vx) = \max_{\vu} \min_{\vxi_x} \:\: r(\vx, \vu) \:\: + f_c(\vx, \vu, \vxi_{x})^{\top} \: \nabla_{\vx} V^{*}, \\
%\hspace{5pt} 
\text{with} \hspace{5pt} \vxi_{x}^{\top} \vxi_{x} - \alpha^{2} \leq 0. 
\end{gather*}
Substituting the dynamics (\Eqref{eq:dyn_adv_state}) model and the reward yields
\begin{align*}
\rho V^{*} &= \max_{\vu} \min_{\vxi_x} \:\: q_c(\vx) - g_c( \vu) \:\: + \left[\va + \mB \vu + \vxi^{*} \right] \: \nabla_x V^{*}, \\
&= \max_{\vu}\left[\vu^{\top} \mB^{\top} \nabla_x V^{*} - g_c(\vu) \right] + \min_{\vxi_x} \left[\vxi_x^{\top} \nabla_x V^{*} \right] \\ 
&\hspace{65pt} + \left[ q_c + \va^{\top}  \nabla_x V^{*} \right].
\end{align*}
Therefore, the optimal action identical to the optimal policy of the \ac{hjb} and is described by 
\begin{gather*}
\vu_{t} = \argmax_{\vu} \:  \nabla_x V^{\top} \mB(\vx_t) \: \vu - g_{c}(\vu) \\
% \hspace{5pt} 
\Rightarrow \hspace{5pt}
\vu_{t} = \nabla \tilde{g}_c\left(\mB(\vx_t) \nabla_x V\right).
\end{gather*}
The optimal state disturbance is described by
\begin{align*}
\vxi^{*}_x = \argmin_{\vxi_x} \: \nabla_x V^{\top} \vxi_x 
\hspace{10pt} \text{with} \hspace{10pt} 
\vxi_{x}^{\top} \vxi_{x} - \alpha^{2} \leq 0.
\end{align*}
This constrained optimization can be solved using the Karush-Kuhn-Tucker (KKT) conditions, i.e., 
\begin{align*}
\nabla_x V + 2 \lambda \: \vxi_x = 0 \hspace{10pt} \Rightarrow \hspace{10pt} \vxi^{*}_{x} = -\frac{1}{2 \lambda} \nabla_x V
\end{align*}
with the Lagrangian multiplier $\lambda \geq 0$. From primal feasibility and the complementary slackness condition of the KKT conditions follows that
\begin{align*}
\frac{1}{4\lambda^{2}} \nabla_x V^{\top} \nabla_x V - \alpha^{2} &\leq 0 \hspace{0pt}  \\
& \hspace{45pt} \Rightarrow \hspace{2pt} \lambda \geq \frac{1}{2 \alpha} \sqrt{\nabla_x V^{\top} \nabla_x V}\\
\lambda \left(\frac{1}{4 \lambda^{2}} \nabla_x V^{\top} \nabla_x V - \alpha^{2}\right) &= 0 \hspace{0pt}\\ 
&\hspace{5pt} \Rightarrow \hspace{2pt} \lambda_0 = 0, \hspace{5pt} \lambda_1 = \frac{1}{2 \alpha} \sqrt{\nabla_x V^{\top} \nabla_x V}.
\end{align*}
Therefore, the optimal adversarial state perturbation with bounded signal energy is described by
\begin{align*}
\vxi^{*} = - \alpha \frac{V^{*}_{x}}{\| V^{*}_{x} \|_2}.
\end{align*}
This solution is intuitive as the adversary wants to minimize the reward and this disturbance performs steepest descent using the largest step-size. Therefore, the perturbation is always on the constraint. \hspace{\fill}\mbox{\qed}

% ##############################################################################
% ##############################################################################
\subsection{Action Disturbance Proof}
\begin{theorem*}
For the adversarial action disturbance (\Eqref{eq:dyn_adv_state}) with bounded in signal energy (\Eqref{eq:signal_energy}), the optimal continuous time policy $\pi^{k}$ and action disturbance $\vxi^{k}_{u}$ is described by
\begin{align*}
\pi^{*}(\vx) &= \nabla \tilde{g} \left(\mB(\vx)^{\top} \nabla_x V^{*}\right) & \vxi^{*}_{u} &= - \alpha \frac{\mB(\vx)^{\top} \nabla_x V^{*}}{\| \mB(\vx)^{\top} \nabla_x V^{*} \|_2}.
\end{align*}
\end{theorem*}

\noindent
\emph{Proof.} The \ac{hji} described in \Eqref{eq:hji} can be formulated with the explicit constraint, i.e., 
\begin{gather*}
\rho \: V^{*}(\vx) = \max_{\vu} \min_{\vxi_u} \:\: r(\vx, \vu) \:\: + f_c(\vx, \vu, \vxi_{u})^{\top} \: \nabla_{\vx} V^{*}, \\
\hspace{5pt} \text{with} \hspace{5pt} \vxi_{u}^{\top} \vxi_{u} \leq \alpha^{2}. 
\end{gather*}
Substituting the dynamics (\Eqref{eq:dyn_adv_state}) model and the reward yields
\begin{align*}
\rho V^{*} &= \max_{\vu} \min_{\vxi_x} \:\: q_c(\vx) - g_c( \vu) \:\: + \left[\va + \mB \left( \vu + \vxi \right) \right] \: \nabla_x V^{*}, \\
&= \max_{\vu}\left[\vu^{\top} \mB^{\top} \nabla_x V^{*} - g_c(\vu) \right] + \min_{\vxi_x} \left[\vxi_u^{\top} \mB^{\top} \nabla_x V^{*} \right] \\ 
&\hspace{65pt} + \left[ q_c + \va^{\top}  \nabla_x V^{*} \right].
%
% &= \max_{\vu}\left[\vu^{\top} \mB^{\top} \nabla_x V^{*} - g_c( \vu) \right] + \min_{\vxi_x} + \left[\vxi_x^{\top} \nabla_x V^{*} \right] \\ 
% & \hspace{122pt} + \left[ q_c(\vx) + \va^{\top}  \nabla_x V^{*} \right].
\end{align*}
Therefore, the optimal action is described by 
\begin{gather*}
\vu = \argmax_{\vu} \:  \nabla_x V^{\top} \mB(\vx) \: \vu - g_{c}(\vu)  \\
% \hspace{5pt} 
\Rightarrow \hspace{2pt}
\vu^{*} = \nabla \tilde{g}_c\left(\mB(\vx) \nabla_x V^{*} \right).
\end{gather*}
The optimal state disturbance is described by
\begin{align*}
\vxi^{*}_u = \argmin_{\vxi_u} \: \nabla_x V^{\top} \mB(\vx) \: \vxi_u \hspace{10pt} \text{with} \hspace{10pt} \vxi_{u}^{\top} \vxi_{u} \leq \alpha^{2}.
\end{align*}
This constrained optimization can be solved using the Karush-Kuhn-Tucker (KKT) conditions, i.e., 
\begin{align*}
\mB(\vx_t)^{\top} \nabla_x V + 2 \lambda \: \vxi_u = 0 \hspace{10pt} \Rightarrow \hspace{10pt} \vxi^{*}_{u} = -\frac{1}{2 \lambda} \mB(\vx_t)^{\top} \nabla_x V
\end{align*}
with the Lagrangian multiplier $\lambda \geq 0$. From primal feasibility and the complementary slackness condition of the KKT conditions follows that
\begin{align*}
\frac{1}{4 \lambda^{2}} \nabla_x V^{\top} \mB \mB^{\top} \nabla_x V - \alpha^{2} &\leq 0 \\
&\hspace{5pt} \Rightarrow \hspace{2pt} \lambda \geq \frac{1}{2 \alpha} \sqrt{\nabla_x V^{\top} \mB \mB^{\top} \nabla_x V}\\
\lambda \left(\frac{1}{4 \lambda^{2}} \nabla_x V^{\top} \mB \mB^{\top} \nabla_x V - \alpha^{2}\right) &= 0 \\
&\hspace{-35pt} \Rightarrow \hspace{2pt} \lambda_0 = 0, \hspace{5pt} \lambda_1 = \frac{1}{2\alpha} \sqrt{\nabla_x V^{\top} \mB \mB^{\top} \nabla_x V}.
\end{align*}
Therefore, the optimal adversarial state perturbation is described by
\begin{align*}
\vxi^{*}_{u} = - \alpha \frac{\mB(\vx_t)^{\top} V^{*}_{x}}{\| \mB(\vx_t)^{\top} V^{*}_{x} \|_2}.
\end{align*}
This solution is similar to the state adversary but in this case the disturbance must be projected to the action space using the control matrix $\mB$. 
\hspace{\fill}\mbox{\qed}

% ##############################################################################
% ##############################################################################
% \newpage
\subsection{Model Disturbance Proof}
\begin{theorem*}
For the adversarial model disturbance (\Eqref{eq:dyn_adv_model}) with element-wise bounded amplitude (\Eqref{eq:amplitude}), smooth drift and control matrix (i.e., $\va, \mB \in C^{1}$) and $\mB(\theta + \vxi_{\theta}) \approx \mB(\theta)$, the optimal continuous time policy $\pi^{k}$ and model disturbance~$\vxi^{k}_{\theta}$ is described by
\begin{gather*}
\pi^{*}(\vx) = \nabla \tilde{g} \left(\mB(\vx)^{\top} \nabla_x V^{*} \right) \hspace{20pt}
\vxi^{*}_{\theta} = -\vDelta_{\nu} \sign \left( \vz_{\theta }\right) + \vmu_{\nu} \\
\text{with} \hspace{5pt} \vz_{\theta} = \left(\frac{\partial \mB}{\partial \theta} \vu^{*} + \frac{\partial \va}{\partial \theta} \right)^{\top} V^{*}_{x},
\end{gather*}
mean $\vmu_{\vnu} = \left( \vnu_{\text{max}} + \vnu_{\text{min}} \right) / 2$ and range $\vDelta_{\vnu} = \left( \vnu_{\text{max}} - \vnu_{\text{min}} \right) / 2$.
\end{theorem*}

\medskip
\noindent
\emph{Proof.} The \ac{hji} described in \Eqref{eq:hji} can be written with the explicit constraint instead of the infimum with the admissible set $\Omega_{A}$. With the bounded amplitude constraint, this constrained optimization is described by
\begin{gather*}
\rho \: V^{*}(\vx) = \max_{\vu} \min_{\vxi_x} \:\: r(\vx, \vu) \:\: + f_c(\vx, \vu, \vxi_{\theta})^{\top} \: \nabla_{\vx} V^{*}, \\
\text{with} \hspace{5pt} \frac{1}{2}\left((\vxi_{\theta} - \vmu_{\nu})^{2} - \vDelta_{\nu}^{2}\right) \leq \mathbf{0}.
\end{gather*}
Substituting the dynamics (\Eqref{eq:dyn_adv_model}), the reward function as well as abbreviating $\mB(\vx; \theta + \vxi_{\theta}) = \mB_{\xi}$ and $\va(\vx; \theta + \vxi_{\theta}) = \va_{\xi}$ yields
\begin{align*}
\rho V^{*} &= \max_{\vu} \min_{\vxi_{\xi}} \:\: q_c(\vx) - g_c( \vu) \:\: + \left[\va_{\xi} + \mB_{\xi} \vu \right]^{\top} \: \nabla_x V^{*}, \\
&= \max_{\vu} \min_{\vxi_x} \left[\left[\va_{\xi} + \mB_{\xi} \vu \right]^{\top} \: \nabla_x V^{*} - g_c(\vu) \right]  
+  q_c.
\end{align*}
Therefore, the action and disturbance is determined by 
\begin{align*}
\vu^{*}, \vxi^{*}_{\theta} = \argmax_{\vu} \argmin_{\vxi} \left[\nabla_x V^{\top} \left(\va_{\xi} + \mB_{\xi} \vu \right) - g_c(\vu) \right].
\end{align*}
This nested max-min optimization can be solved by first solving the inner optimization w.r.t. to $\vu$ and substituting this solution into the outer maximization. The Lagrangian for the optimal model disturbance is described by
\begin{align*}
\vxi^{*} = \argmin_{\vxi} \: \nabla_x V^{\top} \left(\va_{\xi} + \mB_{\xi} \vu \right) + \frac{1}{2} \vlambda^{\top} \left((\vxi_{\theta} - \vmu_{\nu})^{2} - \vDelta_{\nu}^{2}\right)
\end{align*}
Using the KKT conditions this optimization can be solved. The stationarity condition yields
\begin{gather*}
\vz_{\theta} + \vlambda^{\top} (\vxi_{\theta} - \vmu_{\nu}) = 0  \hspace{15pt} \Rightarrow \hspace{15pt} \vxi^{*}_{\theta} = - \vz_{\theta} \oslash \vlambda + \vmu_{\nu} \\
\text{with} \hspace{5pt} \vz_{\theta} =  \left[ \frac{\partial \va}{\partial \theta} + \frac{\partial \mB}{\partial \theta} \vu \right]^{\top} \nabla_x V
\end{gather*}
%
% \begin{gather*}
% \vz_{\theta} + \vlambda^{\top} (\vxi_{\theta} - \vmu_{\nu}) =  \hspace{15pt} \text{with} \hspace{15pt} \vz_{\theta} =  \left[ \frac{\partial \va}{\partial \theta} + \frac{\partial \mB}{\partial \theta} \vu \right]^{\top} \nabla_x V \\ 
% \Rightarrow \hspace{5pt} \vxi^{*}_{\theta} = - \vz_{\theta} \oslash \vlambda + \vmu_{\nu} \\
% \end{gather*}
and the elementwise division $\oslash$. The primal feasibility and the complementary slackness yields
\begin{align*}
\frac{1}{2} \left(-\vz_{\theta}^{2} \oslash \vlambda^{2} - \vDelta_{\nu}^{2} \right) \leq 0 \hspace{5pt} &\Rightarrow \hspace{5pt}  
\vlambda  \geq \| \vz_{\theta} \|_{1} \oslash \vDelta_{\nu} \\
\frac{1}{2} \vlambda^{\top} \left(-\vz_{\theta}^{2} \oslash \vlambda^{2} - \vDelta_{\nu}^{2} \right) = 0 \hspace{5pt} &\Rightarrow \hspace{5pt} \vlambda_{0} = \mathbf{0}, \hspace{5pt} \vlambda_1 = \| \vz_{\theta} \|_{1} \oslash \vDelta_{\nu}.
\end{align*}
Therefore, the optimal model disturbance is described by 
\begin{align*}
\vxi^{*}_{\theta}(\vu) = % -\vDelta \circ \vz_{\theta} \oslash \| \vz_{\theta}\|_1 +  \vmu_{\nu} =
-\vDelta_{\nu} \sign\big( \vz_{\theta}(\vu) \big) + \vmu_{\nu}
\end{align*}
as $\vz_{\theta} \oslash \| \vz_{\theta}\|_1 = \sign(\vz_{\theta})$. Then the optimal action can be computed by 
\begin{align*}
\vu^{*} = \argmax_{\vu} \nabla_x V^{\top} \left[\va(\vxi^{*}_{\theta}(\vu)) + \mB\left(\vxi^{*}_{\theta}(\vu)\right) \vu \right] - g_c(\vu).
\end{align*}
Due to the envelope theorem, the extrema is described by
\begin{align*}
\mB(\vx; \theta + \vxi^{*}(\vu))^{\top} \nabla_x V - g_{c}(\vu) = 0.
\end{align*}
This expression cannot be solved without approximation as $\mB$ does not necessarily be invertible w.r.t. $\theta$. Approximating $\mB(\vx; \theta + \vxi^{*}(\vu)) \approx \mB(\vx; \theta)$, lets one solve for $\vu$. In this case the optimal action $\vu^{*}$ is described by  
$\vu^{*} = \nabla \tilde{g}(\mB(\vx; \theta)^{\top} \nabla_x V)$. This approximation is feasible for two reasons. First of all, if the adversary can significantly alter the dynamics in each step, the system would not be controllable and the optimal policy would not be able to solve the task. Second, this approximation implies that neither agent or the adversary can react to the action of the other and must choose simultaneously. This assumption is common in prior works \cite{bansal2017hamilton}. The order of the minimization and maximization is interchangeable. For both cases the optimal action as well as optimal model disturbance are identical and require the same approximation during the derivation.  
\hspace{\fill}\mbox{\qed}

% ##############################################################################
% ##############################################################################
\medskip
\subsection{Observation Disturbance Proof}
\begin{theorem*}
For the adversarial observation disturbance (\Eqref{eq:dyn_adv_obs}) with bounded signal energy (\Eqref{eq:signal_energy}), smooth drift and control matrix (i.e., $\va, \mB \in C^{1}$) and $\mB(\vx + \vxi_{o}) \approx \mB(\vx)$, the optimal continuous time policy $\pi^{k}$ and observation disturbance~$\vxi^{k}_{o}$ is described by
\begin{gather*}
\pi^{*}(\vx) = \nabla \tilde{g} \left(\mB(\vx)^{\top} \nabla_x V^{*} \right) \hspace{30pt}
\vxi^{*}_{o} = - \alpha \frac{\vz_{o}}{\| \vz_{o} \|_2} \\
\text{with} \hspace{5pt} \vz_{o} = \left( \frac{\partial \va(\vx; \: \theta)}{\partial \vx} + \frac{\partial \mB(\vx; \: \theta)}{\partial \vx} \vu^{*} \right)^{\top} V^{*}_{x}.
\end{gather*}
\end{theorem*}

\medskip
\noindent
\emph{Proof.} The \ac{hji} (\Eqref{eq:hji}) can be written with the explicit constraint instead of the infimum with the admissible set $\Omega_{A}$. The constrained optimization with an bounded energy constraint is described by 
\begin{gather*}
\rho \: V^{*}(\vx) = \max_{\vu} \min_{\vxi_x} \:\: r(\vx, \vu) \:\: + f_c(\vx, \vu, \vxi_{o})^{\top} \: \nabla_{\vx} V^{*}, \\
\text{with} \hspace{5pt} \frac{1}{2}\left(\vxi_{o}^{\top} \vxi_{o} -\alpha^{2}\right) \leq \mathbf{0}.
\end{gather*}
Substituting the dynamics model (\Eqref{eq:dyn_adv_obs}), the reward function as well as abbreviating $\mB(\vx + \vxi_{o}; \theta)$ as $\mB_{\xi}$ yields
\begin{align*}
\rho V^{*} &= \max_{\vu} \min_{\vxi_{\xi}} \:\: q_c(\vx) - g_c( \vu) \:\: + \left[\va_{\xi} + \mB_{\xi} \vu \right]^{\top} \: \nabla_x V^{*}, \\
&= \max_{\vu} \min_{\vxi_x} \left[\left[\va_{\xi} + \mB_{\xi} \vu \right]^{\top} \: \nabla_x V^{*} - g_c(\vu) \right]  
+  q_c.
\end{align*}
Similar to the model disturbance, this nested max-min optimization can be solved by first solving the inner optimization w.r.t. to $\xi$ and substituting this solution into the outer maximization. The Lagrangian for the optimal model disturbance is described by
\begin{align*}
\vxi^{*}_{o} = \argmin_{\vxi_{o}} \: \nabla_x V^{\top} \big(\va(\vxi_{o}) + \mB(\vxi_{o}) \vu \big) + \frac{\lambda}{2} \left(\vxi_{o}^{\top} \vxi_{o} - \alpha^{2}\right)
\end{align*}
Using the KKT conditions this optimization can be solved. The stationarity condition yields
\begin{gather*}
\vz_{o} + \lambda \: \vxi_{o}  = 0  \hspace{15pt} \Rightarrow \hspace{15pt} \vxi^{*}_{o} = - \frac{1}{\lambda} \vz_{o} \\
\text{with} \hspace{5pt} \vz_{o} =  \left[ \frac{\partial \va(\vx; \: \theta)}{\partial \vx} + \frac{\partial \mB(\vx; \: \theta)}{\partial \vx} \vu^{*} \right]^{\top} \nabla_x V.
\end{gather*}
The primal feasibility and the complementary slackness yield
\begin{align*}
\frac{1}{2} \left(\frac{1}{\lambda^{2}} \vz_{o}^{\top}\vz_{o} - \alpha^{2} \right) \leq 0 \hspace{10pt} &\Rightarrow \hspace{10pt} 
\lambda  \geq \frac{1}{\alpha} \| \vz_{\theta} \|_{2}  \\
\frac{\lambda}{2}  \left(\frac{1}{\lambda^{2}} \vz_{o}^{\top} \vz_{o}  - \alpha^{2} \right) = 0 \hspace{10pt} &\Rightarrow \hspace{10pt} \lambda_{0} = 0, \hspace{5pt} \lambda_1 = \frac{1}{\alpha} \| \vz_{o} \|_{2}.
\end{align*}
Therefore, the optimal observation disturbance is described by 
\begin{align*}
\vxi^{*}_{o}(\vu) = -\alpha \frac{\vz_{o}}{\| \vz_{o} \|_2}.
\end{align*}
Then the optimal action can be computed by 
\begin{align*}
\vu^{*} = \argmax_{\vu} \nabla_x V^{\top} \left[\va\big(\vxi^{*}_{o}(\vu)\big) + \mB\left(\vxi^{*}_{o}(\vu)\right) \vu \right] - g_c(\vu).
\end{align*}
Due to the envelope theorem, the extrema is described by
\begin{align*}
\mB(\vx; \theta + \vxi^{*}_{o}(\vu))^{\top} \nabla_x V - g_{c}(\vu) = 0.
\end{align*}
This expression cannot be solved without approximation as $\mB$ does not necessarily be invertible w.r.t. $\vx$. Approximating $\mB(\vx + \vxi^{*}_{o}(\vu); \: \theta) \approx \mB(\vx; \theta)$, lets one solve for $\vu$. In this case the optimal action $\vu^{*}$ is described by  
$\vu^{*} = \nabla \tilde{g}(\mB(\vx; \theta)^{\top} \nabla_x V)$. This approximation is feasible for two reasons. First of all, if the adversary can significantly alter the dynamics in each step, the system would not be controllable and the optimal policy would not be able to solve the task. Second, this approximation implies that neither agent or the adversary can react to the action of the other and must choose simultaneously. This assumption is common in prior works \cite{bansal2017hamilton}. The order of the minimization and maximization is interchangeable. For both cases the optimal action as well as optimal model disturbance are identical and require the same approximation during the derivation. \hspace{\fill}\mbox{\qed}

\end{document}